\documentclass{article} 

\usepackage{amsmath,amsfonts,bm}

\def\eqref#1{equation~\ref{#1}}

\def\1{\bm{1}}

\DeclareMathAlphabet{\mathsfit}{\encodingdefault}{\sfdefault}{m}{sl}
\SetMathAlphabet{\mathsfit}{bold}{\encodingdefault}{\sfdefault}{bx}{n}

\newcommand{\R}{\mathbb{R}}

\DeclareMathOperator*{\argmax}{arg\,max}
\DeclareMathOperator*{\argmin}{arg\,min}

\usepackage[utf8]{inputenc} 
\usepackage[T1]{fontenc}    
\usepackage{hyperref}       
\usepackage{url}            
\usepackage{booktabs}       
\usepackage{amsfonts}       
\usepackage{nicefrac}       
\usepackage{microtype}      
\usepackage{overpic} 
\usepackage{amsmath,amssymb,bbm,stmaryrd,mathrsfs,amsthm}
\usepackage{wrapfig}
\usepackage{mwe,color,xcolor,float}

\newcommand{\E}{\mathbb{E}}

\newcommand{\T}{\mathrm{T}}

\newcommand{\rank}{\operatorname{rank}}

\renewcommand{\sl}{s_\ell}

\newcommand{\setN}{\mathcal{N}}

\usepackage{xr}
\usepackage[export]{adjustbox}
\usepackage{graphicx}
\usepackage{subcaption}
\usepackage{multirow,bigdelim}
\usepackage{blkarray}
\usepackage[all,cmtip]{xy}
\usepackage{amstext}
\usepackage{array}
\usepackage{caption,subcaption}
\usepackage{makecell}

\usepackage[accepted]{icml2020}

\usepackage{makeidx}
\makeindex

\usepackage{xr}
\newtheorem{thm}{Theorem}
\newtheorem{lem}{Lemma}

\newcommand{\diag}{\operatorname{diag}}

\newcommand{\Pcal}{\operatorname{\mathcal{P}}}

\icmltitlerunning{Invertible generative models for inverse problems: mitigating representation error and dataset bias}

\begin{document}

\twocolumn[

\icmltitle{Invertible generative models for  inverse problems: mitigating representation error and dataset bias}

\icmlsetsymbol{equal}{*}
\icmlsetsymbol{equal2}{$\dagger$}

\begin{icmlauthorlist}
\icmlauthor{Muhammad Asim}{equal,itu} 
\icmlauthor{Mara Daniels}{equal,neumcs}
\icmlauthor{Oscar Leong}{rice}
\icmlauthor{Ali Ahmed}{equal2,itu}
\icmlauthor{Paul Hand}{equal2,neumcs}  
\end{icmlauthorlist}

\icmlaffiliation{itu}{Department of Electrical Engineering, Information Technology University, Lahore, Pakistan}
\icmlaffiliation{rice}{Department of Computational and Applied Mathematics, Rice University, Houston, TX}
\icmlaffiliation{neumcs}{Department of Mathematics and Khoury College of Computer Sciences, Northeastern University, Boston, MA}

\icmlcorrespondingauthor{Mara Daniels}{daniels.g@northeastern.edu}
\icmlkeywords{Inverse Problem, GAN, Invertible Neural Network}

\vskip 0.3in
]

\printAffiliationsAndNotice{\icmlEqualContribution}

\begin{abstract}
	Trained generative models have shown remarkable performance as priors for inverse problems in imaging -- for example, Generative Adversarial Network priors permit recovery of test images from 5-10x fewer measurements than sparsity priors.  Unfortunately, these models may be unable to represent any particular image because of architectural choices, mode collapse, and bias in the training dataset. In this paper, we demonstrate that invertible neural networks, which have zero representation error by design, can be effective natural signal priors at inverse problems such as denoising, compressive sensing, and inpainting.  Given a trained generative model, we study the empirical risk formulation of the desired inverse problem under a regularization that promotes high likelihood images, either directly by penalization or algorithmically by initialization. For compressive sensing, invertible priors can yield higher accuracy than sparsity priors across almost all undersampling ratios, and due to their lack of representation error, invertible priors can yield better reconstructions than GAN priors for images that have rare features of variation within the biased training set, including out-of-distribution natural images.  We additionally compare performance for compressive sensing to unlearned methods, such as the deep decoder, and we establish theoretical bounds on expected recovery error in the case of a linear invertible model.
\end{abstract}

\section{Introduction}

\begin{figure}[H]
	\centering
	\includegraphics[width=\linewidth]{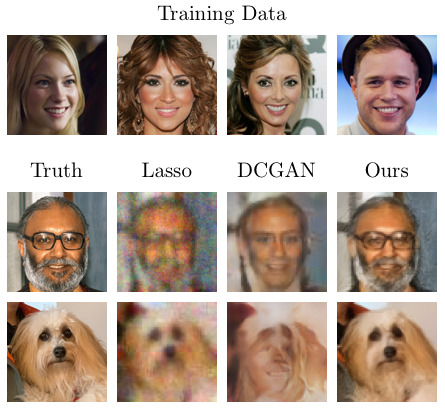}
	\caption{We train an invertible generative model with CelebA images (including those shown). When used as a prior for compressive sensing, it can yield higher quality image reconstructions than Lasso and a trained DCGAN, even on out-of-distribution images. Note that the DCGAN reflects biases of the training set by removing the man's glasses and beard, whereas our invertible prior does not.}
	\label{fig:summary}
\end{figure}
Generative deep neural networks have shown remarkable performance as natural signal priors in imaging inverse problems, such as denoising, inpainting, compressive sensing, blind deconvolution, and phase retrieval.  These generative models can be trained from datasets consisting of images of particular natural signal classes, such as faces, fingerprints, MRIs, and more \citep{karras2017progressive, minaee2018finger, shin2018medical, chen2018efficient}.    Some such models, including variational autoencoders (VAEs) and generative adversarial networks (GANs), learn an explicit low-dimensional manifold that approximates a natural signal class \citep{GAN, kingma2013auto, rezende2014stochastic}.  We will refer to such models as GAN priors.   These priors can be used for inverse problems by attempting to find the signal in the range of the generative model that is most consistent with provided measurements.  When the GAN has a low dimensional latent space, this allows for a low dimensional optimization problem that operates directly on the natural signal class.   Consequently, generative priors can obtain significant performance improvements over classical methods.  For example, GAN priors have been shown to outperform sparsity priors at compressive sensing with 5-10x fewer measurements in some cases. Additionally, GAN priors have led to novel theory for signal recovery in the linear compressive sensing and nonlinear phase retrieval problems \citep{bora2017compressed, hand2017global,hand2018phase}, and they have also shown promising results for the nonlinear blind image deblurring problem \citep{asim2018blind}.

A significant drawback of GAN priors for solving inverse problems is that they can have large representation error or bias due to architecture and training.  That is, a desired image may not be in or near the range of a particular trained GAN.  Representation error can occur both for in-distribution and out-of-distribution images.  For in-distribution, it can be caused by inappropriate latent dimensionality and mode collapse.  For out-of-distribution images, representation error can be large in part because the GAN training process explicitly discourages such images due to the presence of the concurrently trained discriminator network.
For many imaging inverse problems, it is important to be able to recover signals that are out-of-distribution relative to training data.  For example, in scientific and medical imaging, novel objects or pathologies may be expressly sought.  Additionally, desired signals may be out-of-distribution because a training dataset has bias and is unrepresentative of the true underlying distribution.  As an example, the CelebA dataset \citep{liu2015deep} is biased toward people who are young, who do not have facial hair or glasses, and who have a light skin tone.  As we will see, a GAN prior trained on this dataset learns these biases and exhibits image recovery failures because of them. 

Several recent priors have been developed that have lower representation error than GANs.  One class of approaches are unlearned neural network priors, such as the Deep Image Prior and the Deep Decoder \citep{ulyanov2018deep,heckel2018deep}.  These are neural networks that are randomly initialized, and whose weights are optimized at inversion time to best fit provided measurements.  They have practically zero representation error for natural images.  Because they are untrained, there is no training set or training distribution, and hence there is no notion of in- or out-of-distribution natural images.  Another class of approaches include updating the weights of the trained GAN at inversion time in an image adaptive way, such as the IAGAN \cite{hussein2019imageadaptive}.  Such an approach could be interpreted as using the GAN as a warm start for a Deep Image Prior. A further approach is Latent Convolutional Models \cite{athar2018latent}, in which a generative prior is trained using high dimensional latent representations which are structured as the parameters of a randomly initialized convolutional neural network.

In this paper, we study flow-based invertible neural networks as signal priors.  These networks are mathematically invertible (one-to-one and onto) by architectural design \citep{dinh2016density, gomez2017reversible, jacobsen2018revnet, kingma2018glow}.  Consequently, they have zero representation error and are capable of recovering any image, including those significantly out-of-distribution relative to a training set; see Figure \ref{fig:summary}. We call the domain of an invertible generator the latent space, and we call the range of the generator the signal space.  These must have equal dimensionality. 
The strengths of these invertible models include: their architecture allows exact and efficient latent-variable inference, direct log-likelihood evaluation, and efficient image synthesis; they have the potential for significant memory savings in gradient computations; and they can be trained by directly optimizing the likelihood of training images.  
This paper emphasizes an additional strength: \emph{because they lack representation error, invertible models can mitigate dataset bias and improve recovery performance on inverse problems, including for signals that are out-of-distribution relative to training data.}

We present a method for using pretrained generative invertible neural networks as priors for imaging inverse problems.  An invertible generator, once trained, can be used for a wide variety of inverse problems, with no specific knowledge of those problems used during the training process.  As an invertible net permits a likelihood estimate for all images, image recovery can be posed as seeking the highest likelihood image that is consistent with provided measurements.

As a proxy for the image log-likelihood, we pose an optimization of squared data-fit over the latent space under regularization by likelihood of latent representations. In the case of denoising, we explicitly penalize log-likelihood of latent codes, while in compressive sensing and inpainting, regularization is achieved algorithmically. This is due in part to initializion with latent code zero. 

The contributions of this paper are as follows.  We train a generative invertible model using the CelebA dataset.  With this fixed model as a signal prior, we study its performance at multiple inverse problems.
\begin{itemize}
\item For image denoising, invertible neural network priors can yield sharper images with higher PSNRs than BM3D \cite{dabov2007bm3d}.
\item For compressive sensing of in-distribution images, invertible neural network priors can yield higher PSNRs than GANs with low-dimensional latent dimensionality (both DCGAN and PG-GAN), Image Adaptive GANs, sparsity priors, and a Deep Decoder across a wide range of undersampling ratios. 
\item Invertible neural networks exhibit graceful performance decay for compressive sensing on out-of-distribution images.  They can yield significantly higher PSNRs than GANs with low latent dimensionality across a wide range of undersampling ratios.  They can additionally yield higher or comparable PSNRs than the Deep Decoder when there are sufficiently many measurements.   
\item We introduce a likelihood-based theoretical analysis of compressive sensing under invertible generative priors in the case that the generator is linear.  Given $m$ linear measurements of an $n$-dimensional signal, we prove a theorem establishing upper and lower bounds of expected squared recovery error in terms of the sum of the squares of the smallest $n-m$ singular values of the model.
\end{itemize}

\section{Method} \label{sec:method}

We assume that we have access to a pretrained generative Invertible Neural Network (INN), $G: \R^n \to \R^n$. We write $x = G(z)$ and $z = G^{-1}(x)$, where $x \in \R^n$ is an image that corresponds to the latent representation $z \in \R^n$.  We will consider a $G$ that has the Glow architecture introduced in \cite{kingma2018glow}. For a short introduction to the Glow model, see section \ref{sec:glowdetails}.

We consider recovering an image $x$ from possibly-noisy linear measurements of the form,
\[
 y = A x + \eta,
\]
where $A \in \R^{m \times n}$ is a measurement matrix and $\eta \in \R^m$ models noise. 
 Given a pretrained invertible generator $G$, we have access to likelihood estimates for all  images $x\in\R^n$.  Hence, it is natural to attempt to solve the above inverse problem by a maximum likelihood formulation given by
\begin{align*}
\min_{x \in \R^n} \|Ax - y\|^2 - \gamma \log p_G(x),
\end{align*}
where $p_G$ is the likelihood function over $x$ induced by $G$, $\| \cdot \|$ is the Euclidean norm, and $\gamma$ is a hyperparameter.  We have found this formulation to be empirically challenging to optimize, and instead we solve an optimization problem over latent space that encourages high likelihood latent representations. In the case of denoising, we solve
\begin{align}
 \min_{z \in \R^n} \|A G(z) - y \|^2 + \gamma \|z\|^2 \label{eq:denoising-formulation}
\end{align}
In the case of compressive sensing, we fix $\gamma = 0$ and solve
\begin{align}
 \min_{z \in \R^n} \|A G(z) - y \|^2  \label{eq:cs-formulation}
\end{align}
Unless otherwise stated, we initialize both formulations at $z_0=0$. 

The motivation for formulations (\ref{eq:denoising-formulation}) and (\ref{eq:cs-formulation}) is as follows.  As a proxy for the likelihood of an image $x\in \R^n$, we will use the likelihood of its latent representation $z = G^{-1}(x)$.  Because the invertible network $G$ was trained to map a standard normal in $\R^n$ to a distribution over images, the log-likelihood of a latent representation $z$ is proportional to $\|z\|^2$. The model induces a probability distribution over the affine space of images consistent with some given measurements, and so our proxy turns the likelihood maximization task over an affine space in $x$ into the geometric task of finding the point on a manifold in $z$-space that is closest to the origin with respect to the Euclidean norm.   In order to approximate that point, we run a gradient descent in $z$ down the data misfit term starting at $z_0=0$.

In principle, this proxy is imperfect in that some high likelihood latent codes may correspond to low likelihood images. We find that the set of such images has low total probability and they are inconsistent with enough provided measurements. For further discussion of our choice of proxy and initialization at $z_0 = 0$, see the Supplemental Materials.

In all experiments that follow, we use an invertible Glow model,  as in \cite{kingma2018glow}. Due to computational considerations, we run most of our experiments on $64 \times 64$px color images with the pixel values scaled between $[0,1]$. For some compressive sensing experiments, we additionally trained a $128 \times 128$px Glow model in order to replicate results at this larger size. Once trained, the Glow prior is fixed for use in each of the inverse problems below.

For comparison to GAN architectures, we train a $64 \times 64$ DCGAN architecture \cite{radford2015unsupervised} and a $128 \times 128$ PGGAN architecture \cite{karras2017progressive}. To use these priors in inverse problems, we use the formulation from \cite{bora2017compressed}, which is the formulation above in the case where the optimization is performed over $\R^k$, $\gamma=0$, and initialization is selected randomly. For comparison to an unlearned neural image prior, we implement an overparameterized variant of a Deep Decoder prior at both resolutions as in \cite{heckel2018deep}. In all experiments, we solve (\ref{eq:cs-formulation}) using L-BFGS for Glow, and Adam \citep{kingma2014adam} for the Deep Decoder, DCGAN, and PGGAN architectures. DCGAN and PGGAN results are reported for an average of 3 runs because we observed some variance due to random initialization. To measure the quality of recovered images, we use Peak Signal-to-Noise Ratio (PSNR). For more information on the training algorithms, hyperparameters, and parameter counts for each of the tested models, see the Supplemental Materials.

\subsection{Details of the Glow Architecture} \label{sec:glowdetails}
The Glow architecture \cite{kingma2018glow} belongs to the class of \textit{normalizing flow models}. A normalizing flow models output signals using a composition of many flow steps which are each individually invertible. In the Glow model, flow steps use an Affine Coupling layer, in which half of the input data is used to determine the scale and translation parameters of an affine transformation applied to the other half of the input data. This operation is shown schematically in Figure \ref{fig:affine-coupling}. Each affine transformation is invertible and has an upper-triangular Jacobian, making it computationally tractable to compute the Jacobian determinant of the entire normalizing flow. In turn, given a simple prior over the latent space, the model can be efficiently trained to sample structured, high-dimensional data by directly maximizing likelihood of the training data.

To ensure that each input component can affect each output component, the Glow models incorporate a pixelwise reshuffling. In other normalizing flows, such as RealNVP \cite{dinh2016density}, this is achieved by a fixed permutation, whereas in Glow it is achieved by a learned 1x1 convolution. Both models consist of multiscale achitectures based on affine coupling and pixelwise reshuffling. We refer the reader to \cite{dinh2016density} and \cite{kingma2018glow} for more details.

\begin{figure}[t]
    \vspace{0.1in}
    \centering
       \includegraphics[width=\linewidth]{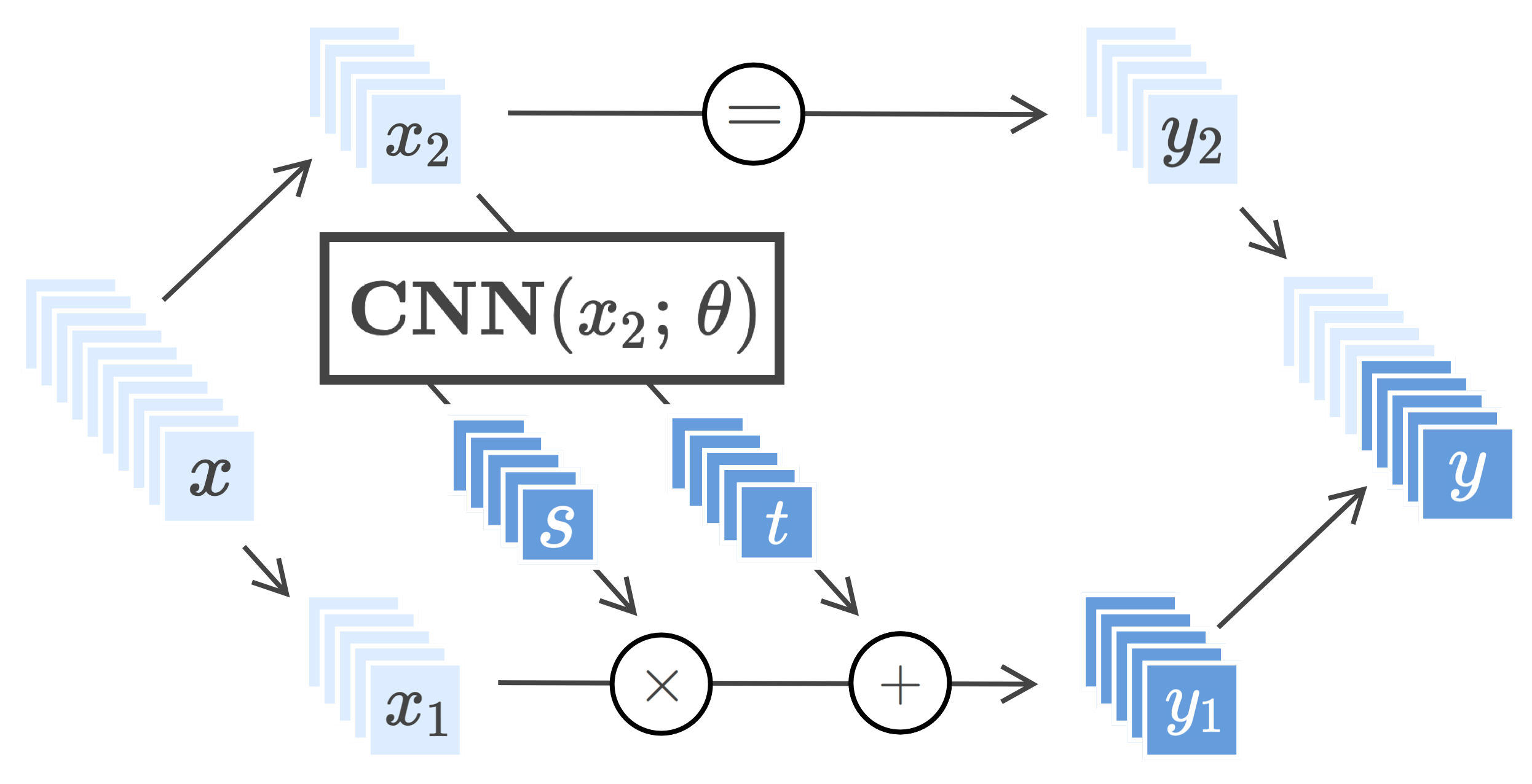}
    \caption{An Affine Coupling layer applies an affine transformation to half of the input data, here $x_1$. The parameters of the affine transformations, $s$ and $t$, can depend in a complex, learned way on the other half of the input data. The model can be inverted, even though $s$ and $t$ themselves are not invertible.}
    \label{fig:affine-coupling}
\end{figure}
\section{Applications}

\subsection{Denoising}
We consider the denoising problem with $A = I_n$ and $\eta \sim \mathcal{N}(0, \sigma^2 I_n)$, as given by formulation \eqref{eq:denoising-formulation}.  We evaluate the performance of a Glow prior, a DCGAN prior, and BM3D for two different noise levels on $64$px in-distribution images ($n = 64 \times 64 \times 3 = 12288$).

Figure \ref{fig:denosing-psnr-vs-gamma} shows the recovered PSNR values as a function of $\gamma$ for denoising by the Glow and DCGAN priors, along with the PSNR by BM3D.  The figure shows that the performance of the regularized Glow prior increases with $\gamma$, and then decreases. If $\gamma$ is too low, then the network fits to the noise in the image.  If $\gamma$ is too high, then data fit is not enforced strongly enough. We study this effect for an extensive range of $\gamma$ and noise levels, which may be found in the Supplemental Materials. We see in Figure \ref{fig:denosing-psnr-vs-gamma} that an appropriately regularized Glow prior can outperform BM3D by almost 2 dB.   The experiments also reveal that appropriately regularized Glow priors outperform the DCGAN prior, which suffers from representation error and is not aided by the regularization.
\begin{figure}[t]
    \centering
       \includegraphics[width=0.727\linewidth]{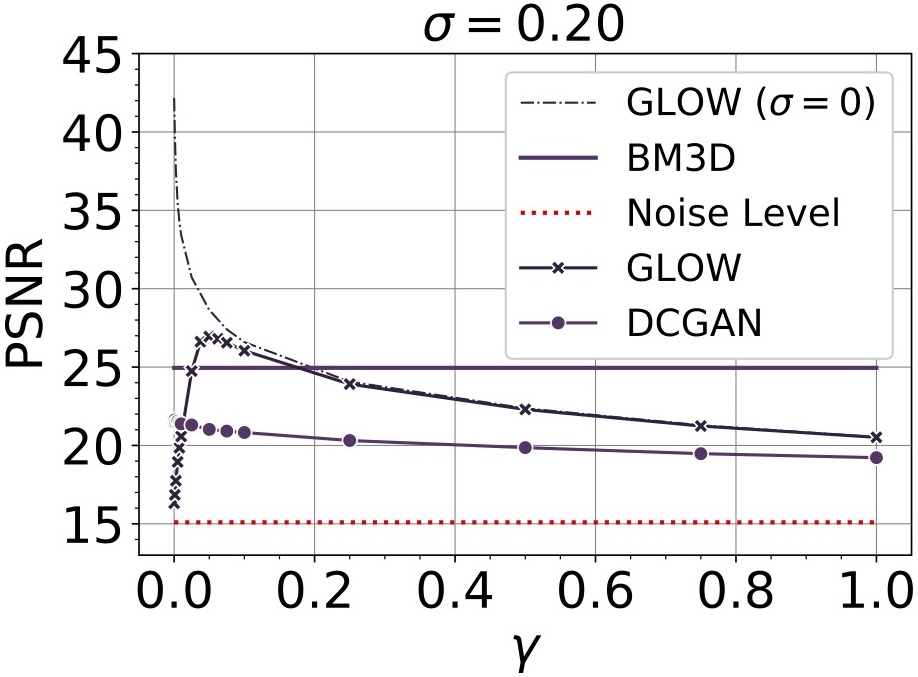}
    \caption{Recovered PSNR values as a function of $\gamma$ for denoising by the Glow and DCGAN priors.  Denoising results are averaged over $N=50$ in-distribution test set images. For reference, we show the average PSNRs of the original noisy images, after applying BM3D, and under the Glow prior in the noiseless case ($\sigma = 0$). }
    \vspace{-0.1in}
    \label{fig:denosing-psnr-vs-gamma}
\end{figure}
A visual comparison of the recoveries at the noise level $\sigma = 0.1$ using Glow, DCGAN priors, and BM3D can be seen in Figure \ref{fig:denoising-vs-gamma-visually}.  Note that the recoveries with Glow are sharper than BM3D. See the Supplemental Materials for more quantitative and qualitative results.

\begin{figure}[h]
    \centering
     \includegraphics[width=0.95\linewidth]{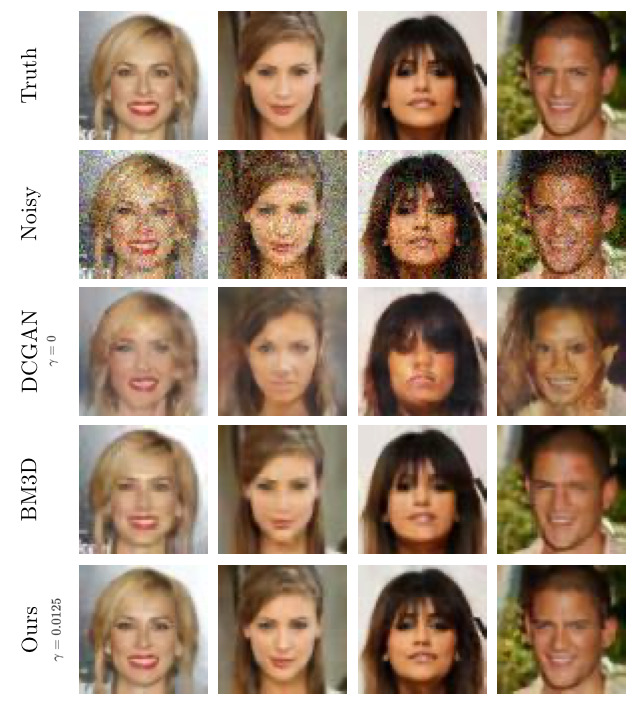}
     \vspace{-0.1in}
     \caption{Denoising results using the Glow prior, the DCGAN prior, and BM3D at noise level $\sigma = 0.1$. Note that the Glow prior gives a sharper image than BM3D in these cases.}
    \vspace{-0.15in}
    \label{fig:denoising-vs-gamma-visually}
\end{figure}

\subsection{Compressive Sensing}

In compressive sensing, one is given undersampled linear measurements of an image, and the goal is to recover the image from those measurements.  In our notation, $A \in \R^{m \times n}$ with $m < n$. 
As the image $x$ is undersampled, there is an affine space of images consistent with the measurements, and an algorithm must select which is most `natural.'  A common proxy for naturalness in the literature has been sparsity with respect to the DCT or wavelet bases.  With a GAN prior, an image is considered natural if it lies in or near the range of the GAN.  For an invertible prior, image likelihood is a proxy for naturalness, and under our proxy for likelihood, we consider an image to be natural if it has a latent representation of small norm.

\begin{figure*}[h]
    \centering
    \includegraphics[width=\textwidth]{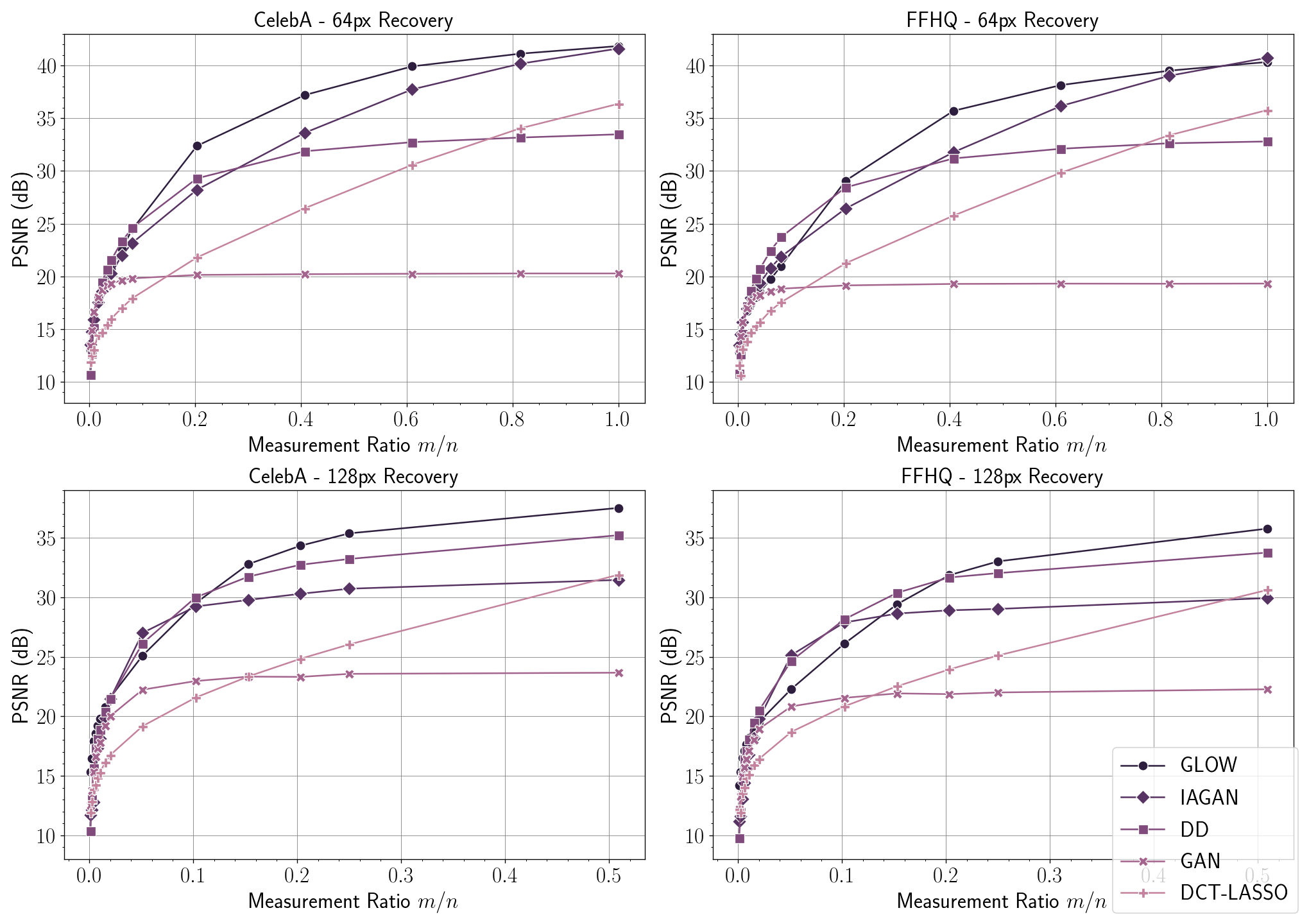}
    \caption{Performance of Glow, GAN, and IAGAN priors (learned) and the Deep Decoder and Lasso-DCT priors (unlearned) across various undersampling ratios in the $64$px and $128$px case. The $64$px and $128$px experiments use $N=1000$ and $N=100$ test set images respectively.} 
    \label{fig:psnrs-summary}
    \vspace{-0.1in}
\end{figure*}

We study compressive sensing in the case that $A$ is an $m \times n$ matrix of i.i.d.  $\setN(0,1/m)$ entries and where $\eta$ is standard i.i.d. Gaussian noise normalized such that $\sqrt{\E\|\eta\|^2}=0.1$. We present our results under formulation (\ref{eq:cs-formulation}), the $\gamma=0$ simplification of (\ref{eq:denoising-formulation}).

We compare the Glow prior to various other learned and unlearned image priors. This includes GANs, as used in \cite{bora2017compressed}, and IAGAN, as used in \cite{hussein2019imageadaptive}. In the $64$px case ($n = 64 \times 64 \times 3 = 12288$), we compare to a DCGAN, and in the $128$px case ($n = 128 \times 128 \times 3 = 49152$) we compare to a PGGAN, both trained on the CelebA-HQ dataset. We also compare to unlearned image priors, including an overparameterized Deep Decoder and a sparsity prior in the DCT basis\footnote{The inverse problems with Lasso were solved by $\min_z \|A \Phi z - y\|_2^2 + 0.01 \|z\|_1$ using coordinate descent. We observe similar performance between the DCT basis and a Wavelet basis. }. To assess the performance of each image prior, we report the mean PSNR of recovered test set images from both the training distribution of the learned models ("in-distribution" images) and other datasets ("out-of-distribution" images). Our in-distribution images are sampled from a test set of CelebA-HQ images which were withheld from all learned models during their training procedures. Out-of-distribution images are sampled randomly from the Flickr Faces High Quality dataset, which provides images with features of variation that are rare among CelebA images (eg. skin tone, age, beards, and glasses) \cite{karras2018stylebased}. The $64$px and $128$px recovery experiments have test sets with $N=1000$ and $N=100$ images respectively.

Surprisingly, the Glow prior exhibits superior performance on compressive sensing tasks with no likelihood penalization in the objective (\ref{eq:cs-formulation}). We find that $z_0 = 0$ is a particularly good choice of initialization, for which one does not benefit from likelihood penalization, while for other choices of initialization direct penalization of likelihood may improve performance. We provide additional experiments exploring this phenomena in the Supplemental Materials.

As shown in Figure \ref{fig:psnrs-summary}, we find that on its training distribution, the Glow prior outperforms both the learned and unlearned alternatives for a wide range of undersampling ratios. Surprisingly, in the case of extreme undersampling, Glow substantially outperforms these methods even though it does not maintain a direct low-dimensional parameterization of the signal manifold. In both the $64$px and $128$px cases, the GAN architectures quickly saturate due to their representation error. For out-of-distribution images, the Glow prior exhibits graceful performance decay, and is still highly performant in a large measurement regime. See figures \ref{fig:in-distribution} and \ref{fig:near-distribution} for a visual comparison of recovered images from the CelebA and FFHQ test sets for the $128$px case. The PGGAN's performance reveals biases of the underlying dataset and limitations of low-dimensional modeling, as the PGGAN fails completely to represent features like darker skin tones or accessories, which are uncommon in CelebA. In contrast, the Glow prior mitigates this bias, demonstrating image recovery for natural images that are not representative of the CelebA training set.
\begin{figure}[h]
    \centering
    \includegraphics[width=\linewidth]{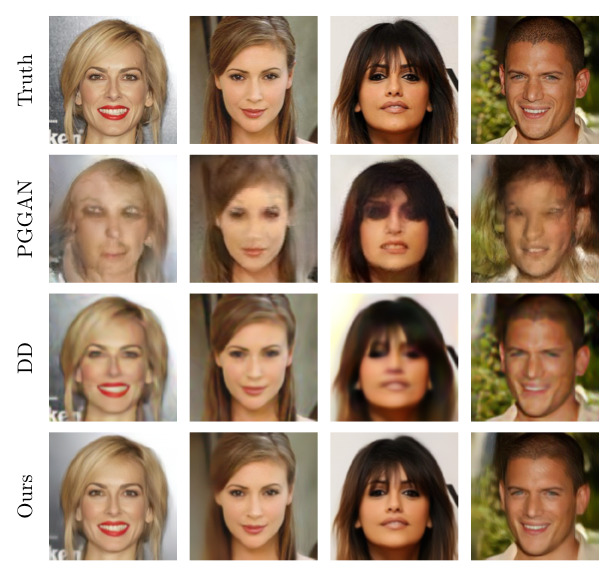}
    \caption{Compressive sensing on CelebA images with $m = 7,500$ ($\approx 20\%$) of measurements. Visual comparisons: CS under the Glow prior, PGGAN prior, and the overparameterized Deep Decoder prior. For images in-distribution, we observe qualitatively sharper recoveries from the Glow Prior than from the Deep Decoder.}
    \label{fig:in-distribution}
\end{figure}

\begin{figure}[h]
    \centering
    \includegraphics[width=\linewidth]{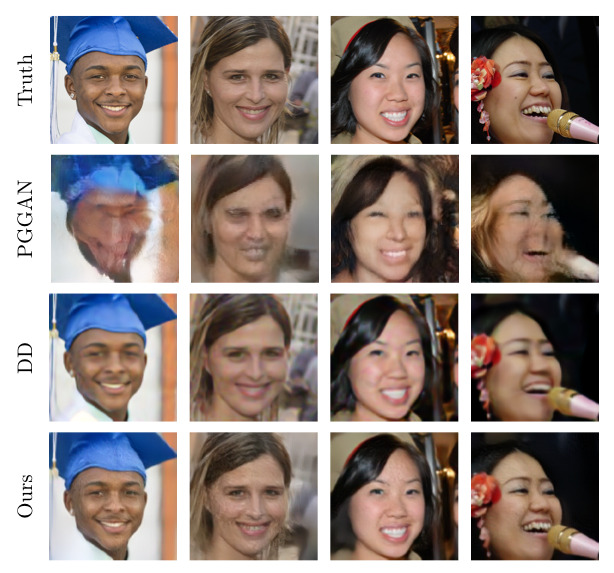}
    \caption{Compressive sensing on FFHQ images with $m = 7,500$ ($\approx 20\%$) of measurements. Visual comparisons: CS under the Glow prior, PGGAN prior, and the overparameterized Deep Decoder prior. For images out-of-distribution, the images recovered by the Deep Decoder and the Glow priors are both qualitatively and quantitatively (by PSNR) comparable.}
    \label{fig:near-distribution}
\end{figure}

\subsection{Inpainting}

In inpainting, one is given a masked image of the form $y = M \odot x$, where $M$ is a masking matrix with binary entries and $x \in \R^n$ is an $n$-pixel image. As in compressive sensing, we solve (\ref{eq:cs-formulation}) to try to recover $x$, among the affine space of images consistent with the measurements.  
\begin{figure}[h]
    \centering
    \includegraphics[width=\linewidth]{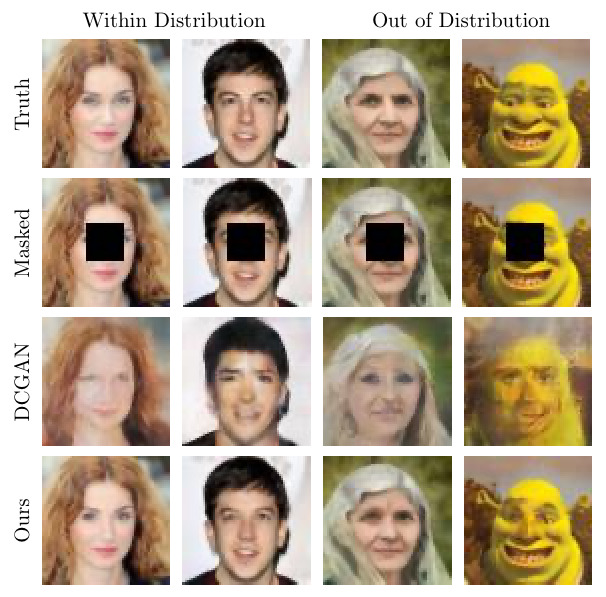}
    \caption{Inpainting on random images with the Glow and DCGAN priors. In the left columns, images are taken from the training distribution of the learned priors, and the right column includes two images from the wild.}
    \label{fig:inpainting}
\end{figure}
As before, using the minimizer $\hat{z}$, the estimated image is given by $G(\hat{z})$. We show qualitative inpainting results in Fig. \ref{fig:inpainting}. Our experiments reveal the same story as forcompressive sensing. If initialized at $z_0=0$, then Glow model under the empirical risk formulation with $\gamma=0$ exhibits high PSNRs on test images while the DCGAN is limited by its representation error.

\section{Theory}

We now introduce a likelihood-based theory for compressive sensing under invertible priors in the case of a linear invertible model.  We will provide an estimate on the expected error of the recovered signal in terms of the singular values of the model.  Specifically, we consider a fixed invertible linear  $G: \R^n \to \R^n$.  Assume signals are generated by a distribution $p_G(x)$, given by $x=G z$, where $z \sim \mathcal{N}(0, I_n)$.  Equivalently, $p_G = \mathcal{N}(0, GG^{\T})$. 

For an unknown sample $x_0\sim p_G$, suppose we are provided noiseless measurements $A x_0$, where $A \in \R^{m \times n}$ has i.i.d. $\mathcal{N}(0,1)$ entries.  We consider the maximum likelihood estimate of $x_0$: 
\begin{align}
    \hat{x} := \argmax_{x \in \R^n}\ p_G(x)\ \text{s.t.}\ Ax = Ax_0. \label{orig_max_likelihood_opt}
\end{align}
The following theorem provides both upper and lower bounds on the absolute expected squared error in terms of the singular values of $G$.

\begin{thm} \label{main_thm}
Suppose $x_0 \sim p_G$ where $ p_G = \mathcal{N}(0,GG^{\T})$ and $G \in \R^{n \times n}$ has singular values $\sigma_1 \geqslant \sigma_2 \geqslant \dots \geqslant \sigma_n > 0$. Let $A \in \R^{m \times n}$ have i.i.d. $\mathcal{N}(0,1)$ entries where $4 \leqslant m < n$. Then the maximum likelihood estimator $\hat{x}$  obeys \begin{align}
    \sum_{i > m} \sigma_i^2 \leqslant \E_A\E_{x_0 \sim p_G}\|\hat{x}-x_0\|^2 \leqslant m \sum_{i > m-2}\sigma_i^2. \label{main_thm_bound}
\end{align}
\end{thm}
The relative expected squared error could be computed by $ \frac{\E_A\E_{x_0 \sim p_G}\|\hat{x}-x_0\|^2}{\E_{x_0 \sim p_G}\|x_0\|^2}$, noting that $\E_{x_0 \sim p_G}\|x_0\|^2 = \sum_{k \geq 1} \sigma_k^2$.

The lower bound of this theorem establishes that under a linear invertible generative model, $m$ Gaussian measurements give rise to at least as much error as would be given by the best $m$-dimensional signal model, i.e. the model corresponding to the only the top $m$ directions of highest variance.  The upper bound on this theorem establishes that up to a factor of $m$, expected square error is bounded above by the error given by the best $m-2$ dimensional model.  

Note that if the singular values decay quickly enough, then the expected recovery error under the linear invertible model decreases to a small value as $m \to n$.  Specifically, this is achieved if $\sigma_i = o(i^{-1/2})$.  
This conclusion is in contrast with the theory for GANs with low latent dimensionality.  In that literature, recovery error does not decrease to 0 as $m \to n$; instead it saturates at the representation error of the GAN model.  In the present case of a linear model, the best $k$-dimensional model would have expected square recovery error at least $\sum_{i > k} \sigma_i^2$ regardless of $m$. Similar work in \cite{yu2011sensinggmms} also showed that $m$ Gaussian measurements of a Gaussian signal give rise to an error proportional to the best $m$-term approximation. The difference between that work and our results is that we have an explicit upfront constant in the upper bound whereas in \cite{yu2011sensinggmms} it is estimated via Monte Carlo simulations.

The theorem is proved in the Supplemental Materials.  

\section{Discussion} 
We have demonstrated that pretrained generative invertible models can be used as natural signal priors in imaging inverse problems.  Their strength is that \emph{every desired image is in the range of an invertible model}, and the challenge that they overcome is that \emph{every undesired image is also in the range of the model and no explicit low-dimensional representation is kept}.  
We demonstrate that this formulation can quantitatively and qualitatively outperform BM3D at denoising.  For compressive sensing on in-distribution images, invertible priors can have lower recovery errors than Deep Decoder, GANs with low dimensional latent representations, and Lasso, across a wide range of undersampling ratios.  We show that the performance of our invertible prior behaves gracefully with slight performance drops for out-of-distribution images.  We additionally prove a theoretical upper and lower bound for expected squared recover error in the case of a linear invertible generative model.  

The idea of analyzing inverse problems with invertible neural networks has appeared in \citet{ardizzone2018analyzing}.  The authors study estimation of the complete posterior parameter distribution under a forward process, conditioned on observed measurements.  Specifically, the authors approximate a particular forward process by training an invertible neural network. The inverse map is then directly available.  In order to cope with information loss, the authors augment the measurements with additional variables.  This work differs from ours because it involves training a separate model for every particular inverse problem.  

In contrast, our work studies how to use a pretrained invertible generator for a variety of inverse problems not known at training time.  Training invertible networks is challenging and computationally expensive; hence, it is desirable to separate the training of off-the-shelf invertible models from potential applications in a  variety of scientific domains. In additional work by \citet{putzky2019invert}, the authors also exploit the efficient gradient calculations of invertible nets for improved MR reconstruction. 

\textit{Why do invertible neural networks perform well for both in-distribution and out-of-distribution images?} \\
One reason that the invertible prior performs so well is because it has no representation error. Any image is potentially recoverable, even if the image is significantly outside of the training distribution.  In contrast, methods based on projecting onto an explicit low-dimensional representation of a natural signal manifold (such as GAN priors) will have representation error, perhaps due to modeling assumptions, mode collapse, or bias in a training set.  Such methods will see performance prematurely saturate as the number of measurements increases.  In contrast, an invertible prior would not see performance saturate.  In the extreme case of having a full set of exact measurements, an invertible prior could in principle recover any image exactly.

\textit{How do invertible priors respect the low dimensionality of natural signals?  And how does our theory inform this?} \\
A surprising feature of invertible priors is that they perform well even though they do not maintain explicit low-dimensional representations of natural signals.  Instead they have a fully dimensional representation of the natural signal class.  Naturally, any signal class will truly be fully dimensional, for example due to sensor noise, but with different importances to different directions.  Those directions in latent space corresponding to noise perturbations will have much smaller of an effect than corresponding perturbations in semantically meaningful directions.    As an illustration, we observe with trained Glow models that the singular values of the Jacobian of $G$ at a natural image exhibit significant decay, as we show in the Supplemental Materials.  In principle, signal models of any given dimensionality could be extracted from $G$, though it is not obvious how to compute these.  \textit{The power of invertible priors is that each additional measurement acts to roughly increment the dimensionality of the modeled natural signal manifold.  That is, more measurements permit exploiting a higher dimensional and, hence, lower-error signal model.  In contrast, GAN-based recovery theory exploits a model of fixed dimensionality regardless of the number of measurements}.  Our theoretical analysis provides justification for this explanation based on linear invertible generators $G$.  Each singular value $\sigma_i$ of $G$ quantifies the signal variation due to a particular direction in latent space, and the expected squared error given $m$ random measurements is upper bounded by $m \sum_{i>m-2} \sigma_i^2$.  Note that the best $m$-dimensional manifold (given by the top $m$ singular values and vectors) would yield at best an expected squared error of $\sum_{i>m}\sigma_i^2$.  The multiplicative $m$ term and the sum's starting index may not be optimal, and improvement of this bound is left for future research.

\begin{figure}[t]
    \centering
    \vspace{0.5cm}
    \includegraphics[width=0.45\linewidth]{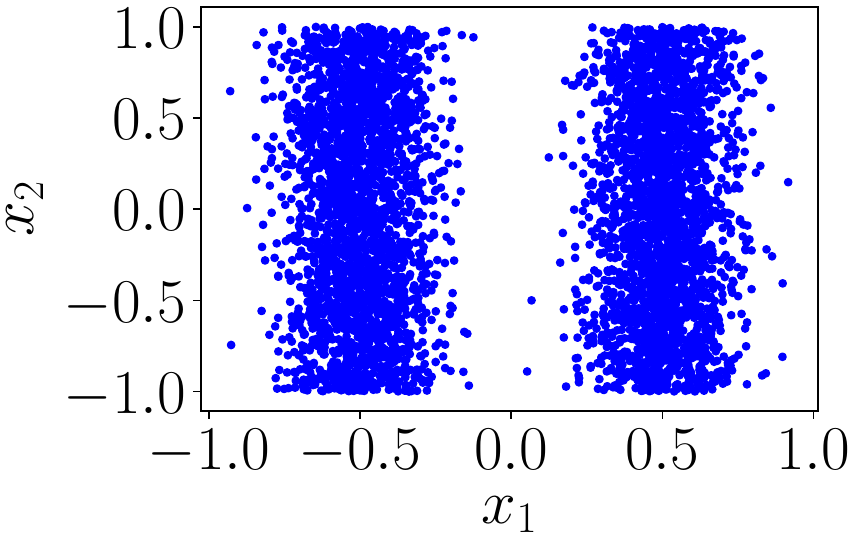} \\
    \includegraphics[width=0.48\linewidth]{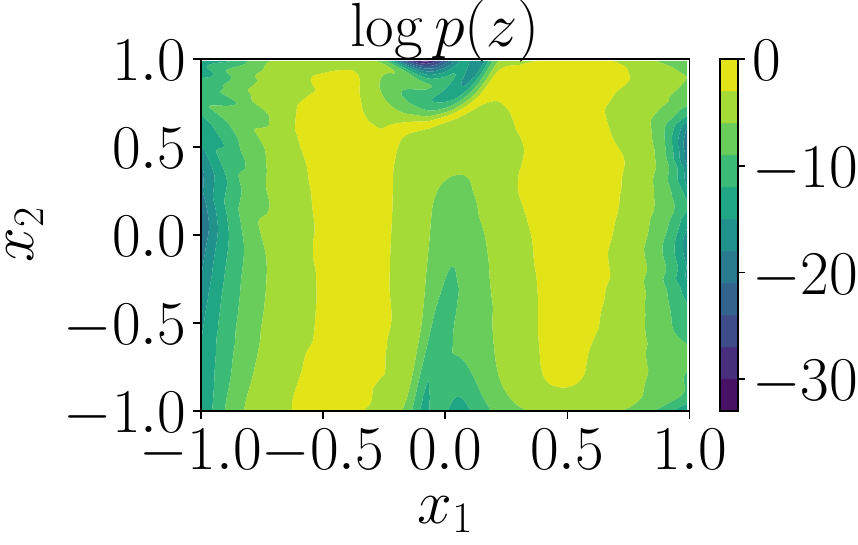} 
    \includegraphics[width=0.48\linewidth]{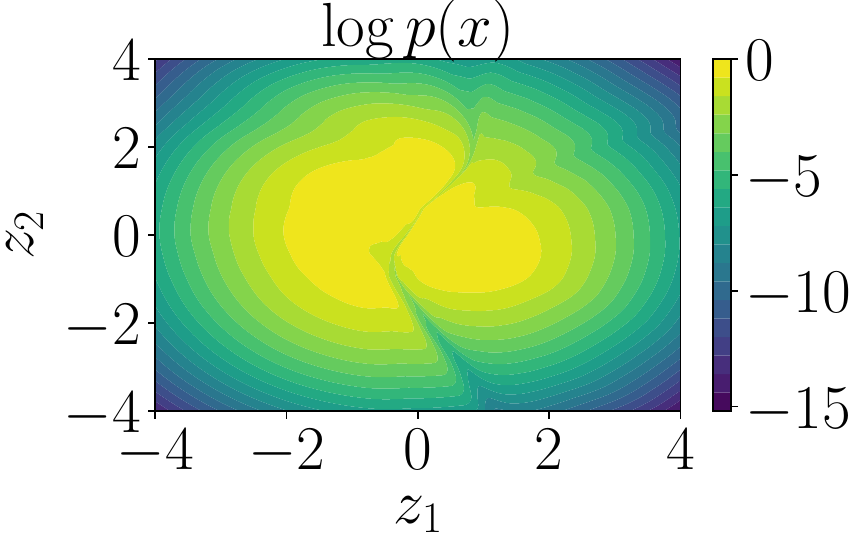}
    \caption{An invertible net was trained on the data points in $x$-space (top), resulting in the given plots of latent $z$-likelihood versus $x$ (bottom left), and $x$-likelihood versus latent representation $z$ (bottom right). }
    \label{fig:2D-example}
    \end{figure}

\textit{Why is the likelihood of an image's latent representation a reasonable proxy for the image's likelihood?} \\
The training process for an invertible generative model attempts to learn a target distribution in image space by directly maximizing the likelihood of provided samples from that distribution, given a standard Gaussian prior in latent space.  High probability regions in latent space map to regions in image space of equal probability.  Hence, broadly speaking, regions of small values of $\|z\|$ are expected to map to regions of large likelihoods in image space.  There will be exceptions to this property.  For example, natural image distributions have a multimodal character.  The preimage of high probability modes in image space will correspond to high likelihood regions in latent space.  Because the generator $G$ is invertible and continuous, interpolation in latent space of these modes will provide images of high likelihood in $z$ but low likelihood in the target distribution.  To demonstrate this, we trained a Real-NVP \citep{dinh2016density} invertible neural network on the two dimensional set of points depicted in Figure \ref{fig:2D-example} (top panel).  The lower left and right panels show that high likelihood regions in latent space generally correspond to higher likelihood regions in image space, but that there are some regions of high likelihood in latent space that map to points of low likelihood in image space and in the target distribution.  We see that the spurious regions are of low total probability and would be unlikely to be the desired outcomes of an inverse problem arising from the target distribution.  

\textit{How can solving compressive inverse problems be successful without direct penalization of the image likelihood or its proxy?} \\
If there are fewer linear measurements than the dimensionality of the desired signal, an affine space of images is consistent with the measurements.  In our formulation, regularization does not occur by direct penalization by our proxy for image likelihood;  instead, it occurs implicitly by performing the optimization in $z$-space with an initialization of $z_0=0$.  The set of latent representations $z$ that are consistent with the compressive measurements define a $m$-dimensional nonlinear manifold.  As per the likelihood proxy mentioned above, the spirit of our formulation is to find the point on this manifold that is closest to the origin with respect to the Euclidean norm.  Our specific way of estimating this point is to perform a gradient descent down a data misfit term in $z$-space, starting at the origin.  While a gradient flow typically will not find the closest point on the manifold, it empirically finds a reasonable approximation of that point. This type of algorithmic regularization is akin to the linear invertible model setting where it is well known that running gradient descent to optimize an underdetermined least squares problem with initialization $z_0=0$ will converge to the minimum norm solution.

The results of this paper provide further evidence that reducing representational error of generators can significantly enhance the performance of generative models for inverse problems in imaging.  This idea was also recently explored in \cite{athar2018latent}, where the authors trained a GAN-like prior with a high-dimensional latent space.  The high dimensionality of this space lowers representational error, though it is not zero.  In their work, the high-dimensional latent space had a structure that was difficult to directly optimize, so the authors successfully modeled latent representations as the output of an untrained convolutional neural network whose parameters are estimated at test time.  

Their paper and ours raises an important question: which generator architectures provide a good balance between low representation error, ease of training, and ease of inversion? This question is important, as solving (\ref{eq:cs-formulation}) in our $128\times 128$px color images experiments took on the order of 11 minutes using an NVIDIA 2080 Ti GPU.

 New developments are needed on architectures and frameworks in between low-dimensional generative priors and fully invertible generative priors.
  Such methods could leverage the strengths of invertible models while being much cheaper to train and use. 

 \section*{Acknowledgements}
  PH is partially supported by NSF CAREER Grant DMS-1848087. OL acknowledges support by the NSF Graduate Research Fellowship under Grant No. DGE-1450681.

\bibliography{GlowIP}
\bibliographystyle{icml2020} 

\onecolumn


\section{Appendix}

\subsection{Proofs}\label{sec:proofs}

We now proceed with the proof of Theorem \ref{main_thm}. Throughout the proof, we use $\stackrel{d}{=}$ to denote equality in distribution.

\begin{proof}[Proof of Theorem \ref{main_thm}] 

We first quantify the expected error over the distribution $p_G$. Observe that since $x_0 \sim p_G = \mathcal{N}(0,GG^{\T})$, it can be written as $x_0 = Gz_0$ where $z_0 \sim \mathcal{N}(0,I_n)$. Then \begin{align*}
    \E_{x_0 \sim p_G} \|\hat{x} - x_0\|^2 = \E_{z_0 \sim \mathcal{N}(0,I_n)} \|G\hat{z} - Gz_0\|^2
\end{align*} where $\hat{z}$ is the maximum likelihood estimator in latent space: \begin{align*}
    \hat{z} := \argmin_{z \in \R^n}\ \frac{1}{2}\|z\|^2\ \text{s.t.}\ AGz = AGz_0.
\end{align*} Since $A$ has i.i.d. $\mathcal{N}(0,1)$ entries, $AG$ has full rank with probability 1 so $\hat{z}$ is given explicitly by\begin{align*}
    \hat{z} = G^{\T}A^{\T}(AGG^{\T}A^{\T})^{-1}AGz_0 =: \Pcal_{G^{\T}A^{\T}} z_0
\end{align*} where $\Pcal_{G^{\T}A^{\T}}$ is the orthogonal projection onto the range of $G^{\T}A^{\T}$. Thus \begin{align*}
     \E_{z_0 \sim \mathcal{N}(0,I_n)} \|G\hat{z} - Gz_0\|^2 = \E_{z_0 \sim \mathcal{N}(0,I_n)}\|G\Pcal_{G^{\T}A^{\T}} z_0 - Gz_0\|^2 = \|G(\Pcal_{G^{\T}A^{\T}}-I_n)\|_F^2
\end{align*} where the last equality follows from Lemma \ref{expected_matrix_times_gaussian}.

Without loss of generality, it suffices to consider the case when $G$ is diagonal. To see this, let $U\Sigma V^{\T}$ be the SVD of $G$. Observe that $$\Pcal_{G^{\T}A^{\T}} = \Pcal_{V\Sigma U^{\T}A^{\T}} = V\Pcal_{\Sigma U^{\T}A^{\T}} V^{\T} \stackrel{d}{=} V\Pcal_{\Sigma A^{\T}} V^{\T}$$ where we used Lemma \ref{proj_commutativity_lemma} in the second equality and used the rotational invariance of $A$ in the last equality. Thus we have \begin{align*}
    \|G(\Pcal_{G^{\T}A^{\T}}-I_n)\|_F^2 \stackrel{d}{=} \|G(V\Pcal_{\Sigma A^{\T}} V^{\T}-I_n)\|_F^2.
\end{align*} Moreover, note that \begin{align*}
    \|G(V\Pcal_{\Sigma A^{\T}} V^{\T}-I_n)\|_F^2  = \|\Sigma V^{\T} (V\Pcal_{\Sigma A^{\T}}V^{\T} - I_n)\|_F^2 = \|\Sigma \Pcal_{\Sigma A^{\T}}V^{\T} - \Sigma V^{\T}\|_F^2 = \|\Sigma (\Pcal_{\Sigma A^{\T}}-I_n)\|_F^2
\end{align*} where we used the unitary invariance of the Frobenius norm in the first and last equality and the orthogonality of $V$ in the second equality. Hence $$\E_A \|G(\Pcal_{G^{\T}A^{\T}}-I_n)\|_F^2 = \E_A\|\Sigma (\Pcal_{\Sigma A^{\T}}-I_n)\|_F^2 = \E_A\|(I_n-\Pcal_{\Sigma A^{\T}})\Sigma\|_F^2$$ so we assume without loss of generality that $G = \diag(\sigma_1,\dots,\sigma_n)$ and consider $\E_A\|(I_n-\Pcal_{GA^{\T}})G\|_F^2$.

We now compute the lower bound of the expected error. Note that $\rank(\Pcal_{GA^{\T}}) = m$ since $A$ is full rank with probability $1$. Also for any random draw of Gaussian $A$, \begin{align}
     \min_{S \in \R^{n\times n},\ \rank(S) \leqslant m} \|G-S\|_F^2 \leqslant \|(I_n-\Pcal_{GA^{\T}})G\|_F^2. \label{best_m_rank_approx_low_bound}
 \end{align} By the Eckart-Young Theorem \cite{EckartYoung},
 \begin{align*}
     \min_{S \in \R^{n\times n},\ \rank(S) \leqslant m} \|G-S\|_F^2 = \sum_{i > m} \sigma_i^2.
 \end{align*} Taking the expectation with respect to $A$ to the right hand-side of (\ref{best_m_rank_approx_low_bound}) establishes the lower bound in the theorem.

 The upper bound comes from the following result in \cite{Troppetal2011}. \begin{thm}[Minor variant of Theorem 10.5 in \cite{Troppetal2011}] \label{tropp_et_al_thm}
Suppose $G$ is a real $\ell \times n$ matrix with singular values $\sigma_1 \geqslant \sigma_2 \geqslant \dots \geqslant \sigma_{\min\{\ell,n\}} \geqslant 0$. Choose a target rank $k \geqslant 2$ and oversampling factor $m - k \geqslant 2$ where $m \leqslant \min\{\ell, n\}$. Draw an $n \times m$ Gaussian matrix $A^{\T}$ and construct the sample matrix $GA^{\T}$. Then the expected approximation error \begin{align}
    \E_A\|(I_n-\Pcal_{GA^{\T}})G\|_F^2 \leqslant \left(1 + \frac{k}{m-k-1}\right)\sum_{j > k}\sigma_j^2. \label{tropp_thm_variant_bound}
\end{align}

\end{thm} \noindent Theorem 10.5 in \cite{Troppetal2011} literally states a bound on $\E_A \|(I_n-\Pcal_{GA^{\T}})G\|_F$ but in the proof, the authors show the stronger result (\ref{tropp_thm_variant_bound}). The upper bound in Theorem \ref{main_thm} follows by setting $\ell = n$ and $k = m-2$ in Theorem \ref{tropp_et_al_thm} whereby the condition on $k$ requires $4 \leqslant m \leqslant n$.

\end{proof}

The following are two supplementary lemmas used in the proof of Theorem \ref{main_thm}. The first shows that the expected value of the $\ell_2$ norm of a matrix acting on an isotropic gaussian vector is its frobenius norm.

\begin{lem} \label{expected_matrix_times_gaussian}
Let $M \in \R^{m \times n}$. Then $\E_{z \sim \mathcal{N}(0,I_n)} \|Mz\|^2 = \|M\|_F^2$.
\end{lem}

\begin{proof} Let $U \Sigma V^{\T}$ be the SVD of $M$ where $\Sigma \in \R^{m \times n}$ contains the singular values $\sigma_i$ of $M$ on the diagonal for $i = 1,\dots,r$ and $r = \rank(M) \leqslant \min\{m,n\}$. Then $$\|Mz\|^2 = \|U\Sigma V^{\T}z\|^2 =  \|\Sigma V^{\T}z\|^2 \stackrel{d}{=} \|\Sigma z\|^2$$ where we used the unitary invariance of the $\ell_2$ norm in the second equality and the rotational invariance of $z$ in the last equality. The result follows by noting that $$\E_{z \sim \mathcal{N}(0,I_n)}\|\Sigma z\|^2 = \sum_{i=1}^{r} \sigma_i^2 \E_{z_i \sim \mathcal{N}(0,1)}z_i^2 = \sum_{i=1}^r \sigma_i^2 = \|M\|_F^2.$$
\end{proof}

The second lemma asserts that unitary matrices exhibit a commutativity property when acting on the range of an orthogonal projector.
\begin{lem}[Proposition 8.4 in \cite{Troppetal2011}] \label{proj_commutativity_lemma}
Given a matrix $M$, let $\Pcal_{M}$ denote the orthogonal projection onto the range of $M$. Then for any unitary $U$, $U^{\T}\Pcal_{M}U = \Pcal_{U^{\T}M}.$
\end{lem}

\newpage
\subsection{Models}
\subsubsection{Architecture and Training Details} \label{sec:arch}
We train $128$px and $64$px variants of the Glow architecture \cite{kingma2018glow}. This model uses a sequence of invertible flow steps, each comprising of an activation normalization layer, a $1 \times 1$ invertible convolutional layer, and an affine couping layer. Let $K$ be the number of steps of flow before each splitting layer, and let $L$ be the number of times the splitting is performed. For denoising, we use $K=48$, $L=4$ for $64$px recovery. For compressive sensing, we use $K=18$ and $L=4$ for $64$px recovery and $K=32$, $L=6$ for $128$px recovery. All models are trained over $5$-bit images using a learning rate of $0.0001$ and $10,000$ warmup iterations, as in \cite{kingma2018glow}. When solving inverse problems using Glow, original $8$-bit images were used.

We observed numerical instability when solving inverse problems in the $128$px case, in that activations of the Glow network could become too large to compute during inversion. To mitigate this problem, we train a modified version of the Glow reference implementation in which we add a small constant $\epsilon = 0.0005$ to computed scale parameters in each Actnorm layer, thus preventing a ``division by zero error'' during inversion. Additionally, in forward passes of the Glow model we clip activations to the range $[-40, 40]$.

We train a $64$px DCGAN \cite{radford2015unsupervised} and a $128$px PGGAN \cite{karras2017progressive} in order to compare the Glow model to traditional GAN architectures. The DCGAN model has $d=5$ upsampling layers implemented through transpose convolutions. The PGGAN model has $d=5$ upsampling layers implemented through nearest neighbor upsampling followed by a convolutional layer. In both cases, the first layer of each GAN includes $k_0=512$ activation channels and after each upsampling operation the number of activation channels is reduced by half. These architectures are used as-is for both the CSGM \cite{bora2017compressed} and the Image Adaptive \cite{hussein2019imageadaptive} compressive sensing procedures.

Lastly, we also use an overparametrized variant of the Deep Decoder \cite{heckel2018deep} in order to compare to an unlearned neural network for compressive sensing. We use $d=5$ and $d=6$ upsampling layers for the $64$px and $128$px cases respectively, implemented in both cases with convolutional layers followed by bilinear upsampling. The number of activation channels in each layer is held fixed at $k=250$ and $k=700$ for the $64$px and $128$px cases respectively.
\begin{figure}[h!]
    \centering 
    \begin{tabular*}{\textwidth}{l @{\extracolsep{\fill}} lllc}
        \toprule
        \textit{$64$px Models} & Hyperparameters & Repr. Size & Model Size & Overparam. Ratio \\ \midrule \midrule
        Glow (Denoising) & $K=48$, $L=4$, affine coupling & 12,288 & 67,172,224 & 1.000 \\ \midrule
        Glow (CS) & $K=18$, $L=4$, affine coupling & 12,288 & 25,783,744 & 1.000 \\ \midrule
        DCGAN &  $k_0=512$, $d=5$ & 100 & 3,576,804  & 0.0081 \\ \midrule
        IA-DCGAN & $k_0=512$, $d=5$ & 3,576,804 & 3,576,804 & 291.1 \\ \midrule
        Deep Decoder & $k=250$, $d=5$ & 253,753 & 254,753 & 20.65 \\ \bottomrule
    \end{tabular*}
    \par\bigskip
    \begin{tabular*}{\textwidth}{l @{\extracolsep{\fill}} lllc}
        \toprule
        \textit{$128$px Models} & Hyperparameters & Repr. Size & Model Size & Overparam. Ratio \\ \midrule \midrule
        Glow & $K=32$, $L=6$, affine coupling & 49152 & 129,451,520 & 1.000 \\ \midrule
        PGGAN & $k_0=512$, $d=5$ & 512 & 13,625,904 & 0.0104 \\ \midrule
        IA-PGGAN & $k_0=512$, $d=5$ & 13,625,904 & 13,625,904 & 277.2 \\ \midrule
        Deep Decoder & $k=700$, $d=6$ & 2,462,603 & 2,473,803 & 50.10 \\ \bottomrule
    \end{tabular*}
    \caption{Summary of the model parameters for all models used in our experiments. Hyperparameters refer to model-specific configurations as described in the text. Representation Size (Repr. Size) is the dimensionality of each model's image representation, i.e. the total number of optimizable parameters available during inversion in each inverse problem. The model size is the total sum of all parameters of each model, including those of the image representation. The Overparametrization Ratio (Overparam. Ratio) describes the representation size as a fraction of the output image dimensionality.}
\end{figure}

\newpage
\subsubsection{Unconditional Output Samples}
To demonstrate the successful training of our generative models, we provide output samples from each of the learned models used in our experiments. 
\begin{figure}[h!]
    \centering
    \begin{subfigure}[h]{0.45\textwidth}
        \centering
        \includegraphics[width=\textwidth]{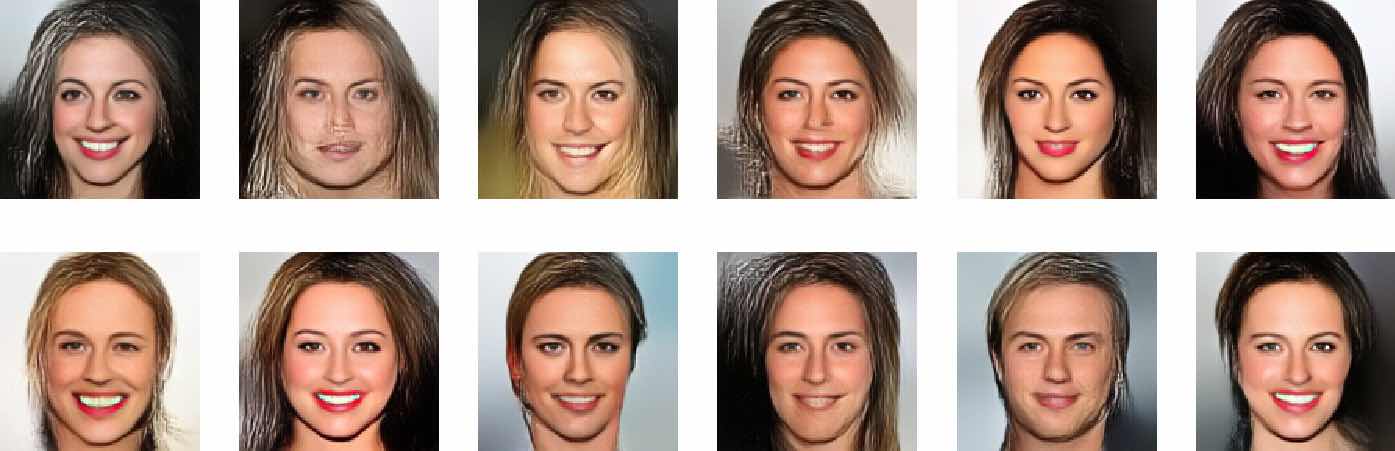}
        \caption{\label{fig:glow_128_samples} $128$px Glow Model}
    \end{subfigure}
    \hspace{0.5cm}
    \begin{subfigure}[h]{0.45\textwidth}
        \centering
        \includegraphics[width=\textwidth]{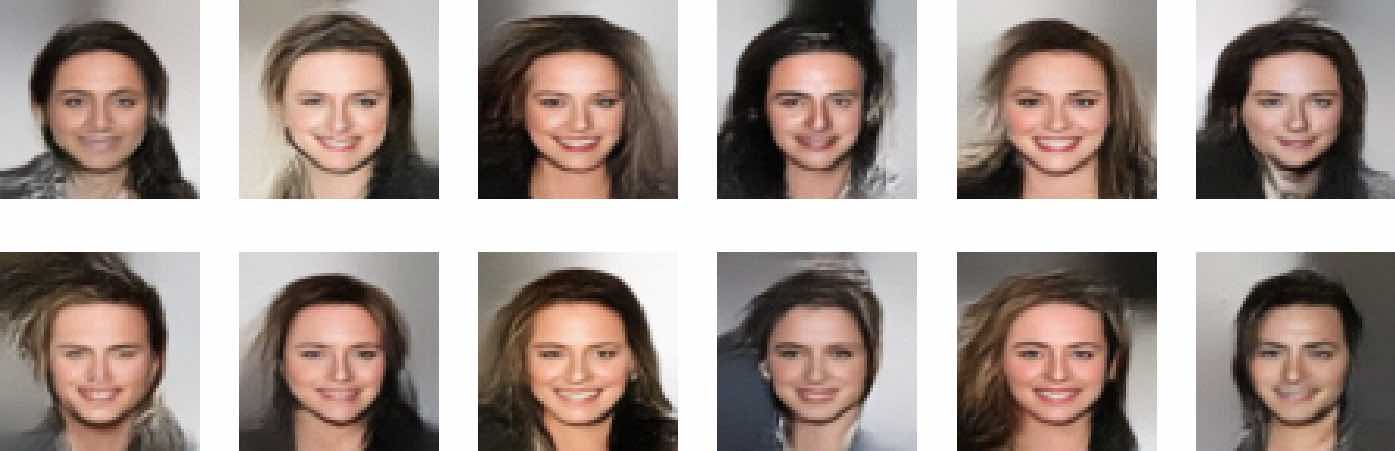}
        \caption{\label{fig:glow_64_samples} $64$px Glow Model}
    \end{subfigure}
    \par\bigskip
    \begin{subfigure}[h]{0.45\textwidth}
        \centering
        \includegraphics[width=\textwidth]{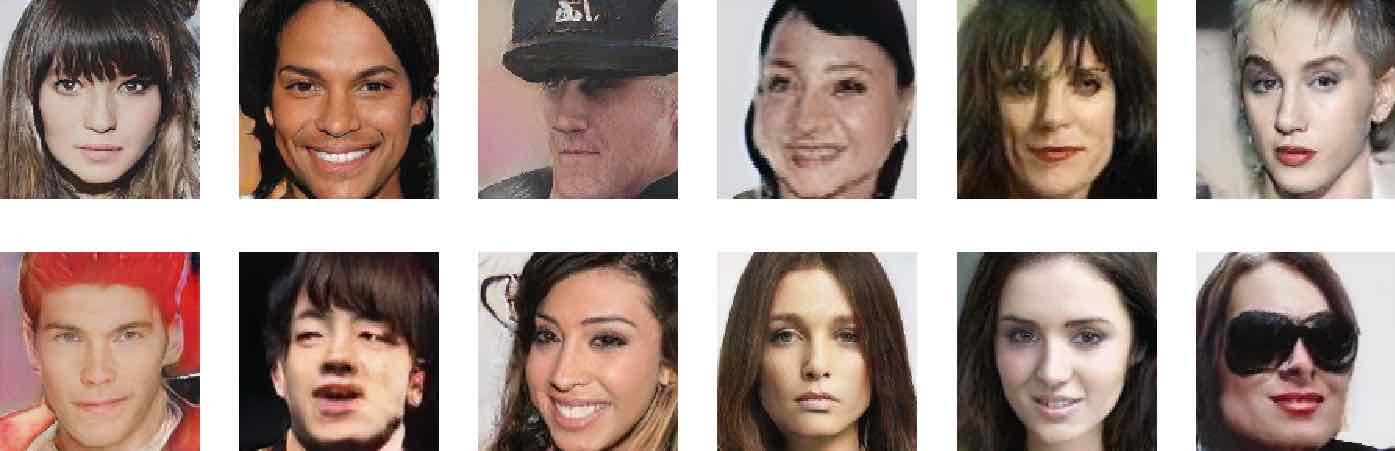}
        \caption{\label{fig:pggan_samples} PGGAN Model}
    \end{subfigure}
    \hspace{0.5cm}
    \begin{subfigure}[h]{0.45\textwidth}
        \centering
        \includegraphics[width=\textwidth]{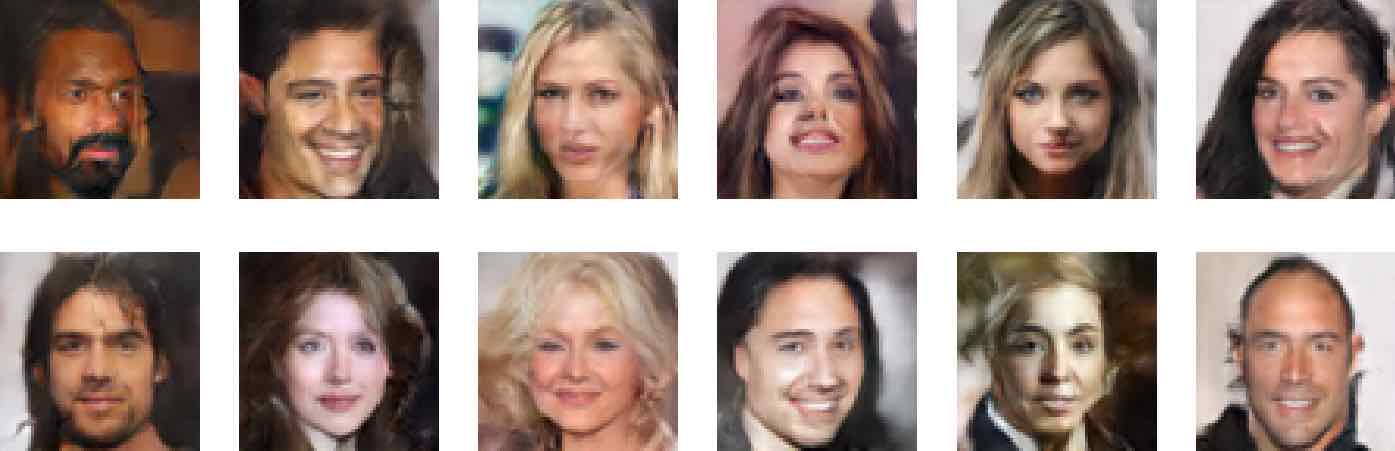}
        \caption{\label{fig:dcgan_samples} DCGAN Model}
    \end{subfigure}

    \caption{Samples from each of our trained generative models.}
\end{figure}
\subsubsection{Test Set Images}
Throughout our paper, we present experiments across two test sets, one from the same distribution of the trained generative models and another which is out-of-distribution. In particular, our in-distribution test set is sampled randomly from a validation split of the CelebA-HQ dataset \cite{karras2017progressive} and our out-of-distribution test set is sampled randomly from the Flickr Faces High Quality (FFHQ) Dataset \cite{karras2018stylebased}. None of the test set images are seen by the generative models during their training phases. A few samples from both test sets are shown below.
\begin{figure}[h!]
    \centering
    \begin{subfigure}[h]{0.45\textwidth}
        \centering
        \includegraphics[width=\textwidth]{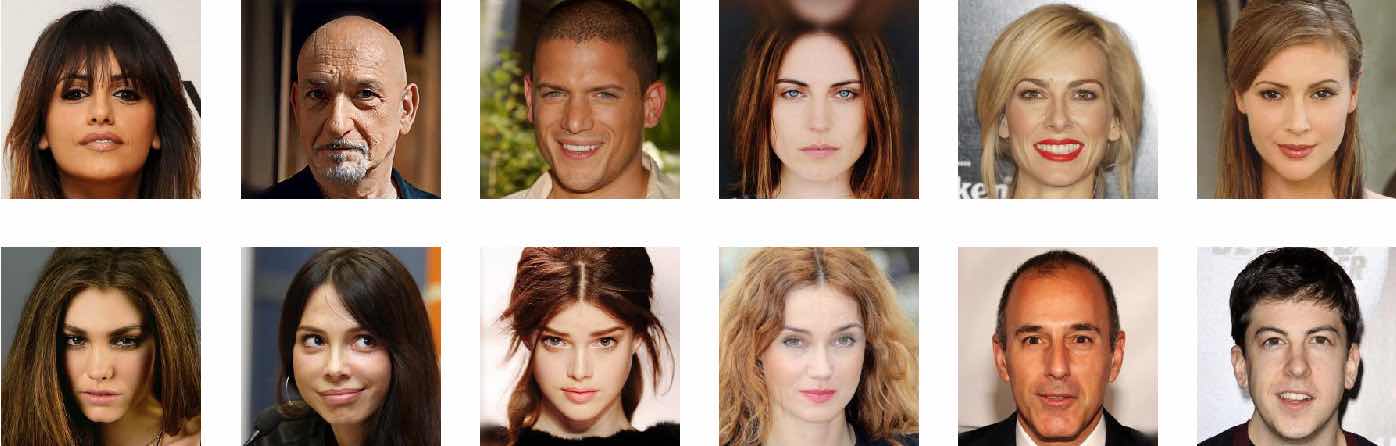}
        \caption{\label{fig:glow_128_samples} CelebA-HQ Test Set}
    \end{subfigure}
    \hspace{0.5cm}
    \begin{subfigure}[h]{0.45\textwidth}
        \centering
        \includegraphics[width=\textwidth]{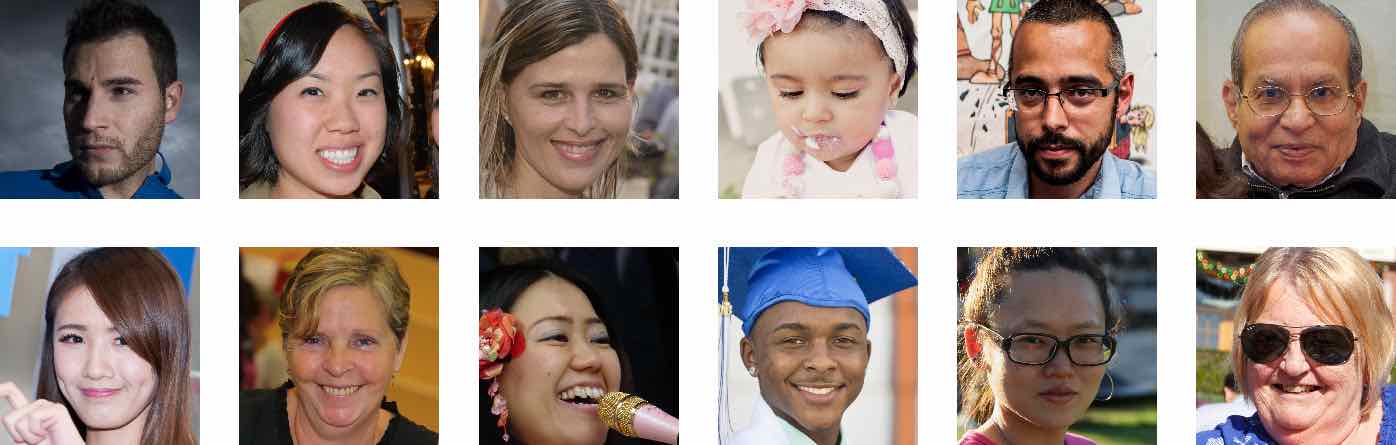}
        \caption{\label{fig:glow_64_samples} FFHQ Test Set}
    \end{subfigure}
    \caption{A few samples from each of the test sets used in our experiments.}
\end{figure}

\newpage 
\subsection{Denoising: Additional Experiments}
\label{sec:appendix-denoising}
We additionally provide qualitative and quantitative denoising results over the test set of in-distribution CelebA images in Figures \ref{fig:denosing-psnr-vs-gamma-all-levels} and \ref{fig:denoising-vs-gamma-visually-higher-noise-2}. Recall that we formulate denoising as the following empirical risk minimization problem:
\begin{align*}
    \min_{z \in \R^n} \|G(z) - y \|^2 + \gamma \|z\|^2 \qquad y = x_0 + \eta
\end{align*}
Where $\eta \sim \mathcal{N}(0, \sigma^2 I_n)$ is additive Gaussian noise. PSNRs for varying choices of the penalization parameter $\gamma$ under noise levels $\sigma = 0.01, 0.05, 0.1,$ and $0.2$ are presented in Figure \ref{fig:denosing-psnr-vs-gamma-all-levels} below.

The central message is that the Glow prior outperforms the DCGAN prior uniformly across all $\gamma$ due to the representation error of the DCGAN prior. In addition, when $\gamma$ is chosen appropriately, regularization improves the performance of Glow, which can outperform the state-of-the-art BM3D algorithm at high noise levels such as $\sigma = 0.2$, and can offer comparable performance at lower noise levels. This is in contrast to the DCGAN prior, whose performance is harmed by increased regularization. 

\begin{figure}[h!]
    \centering
    \hspace{-0.1in}
    \includegraphics[width=0.4\textwidth]{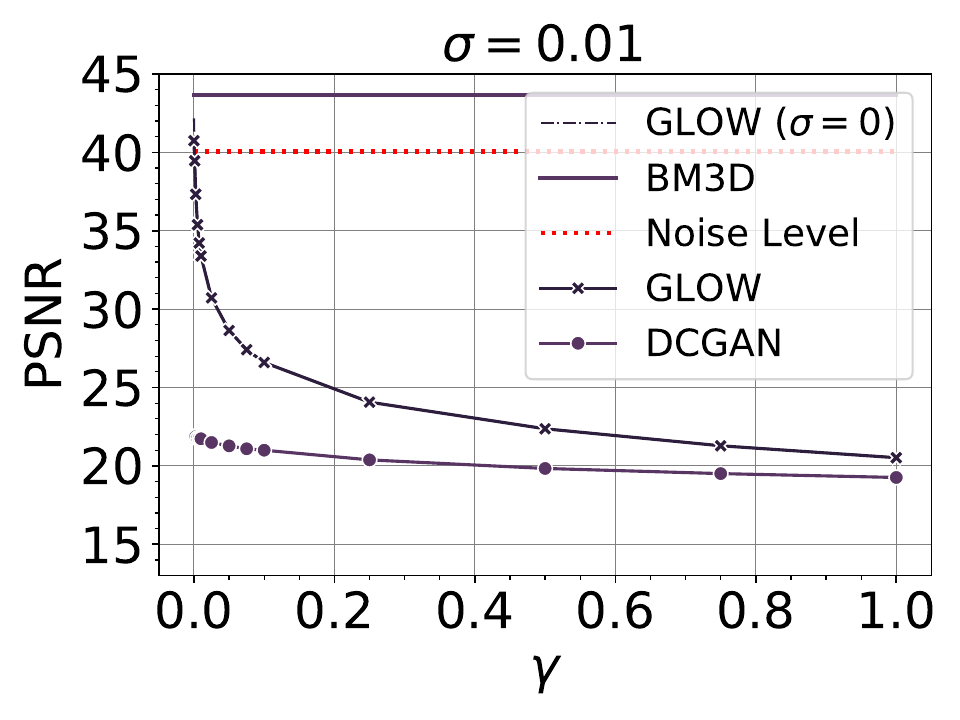}
    \includegraphics[width=0.4\textwidth]{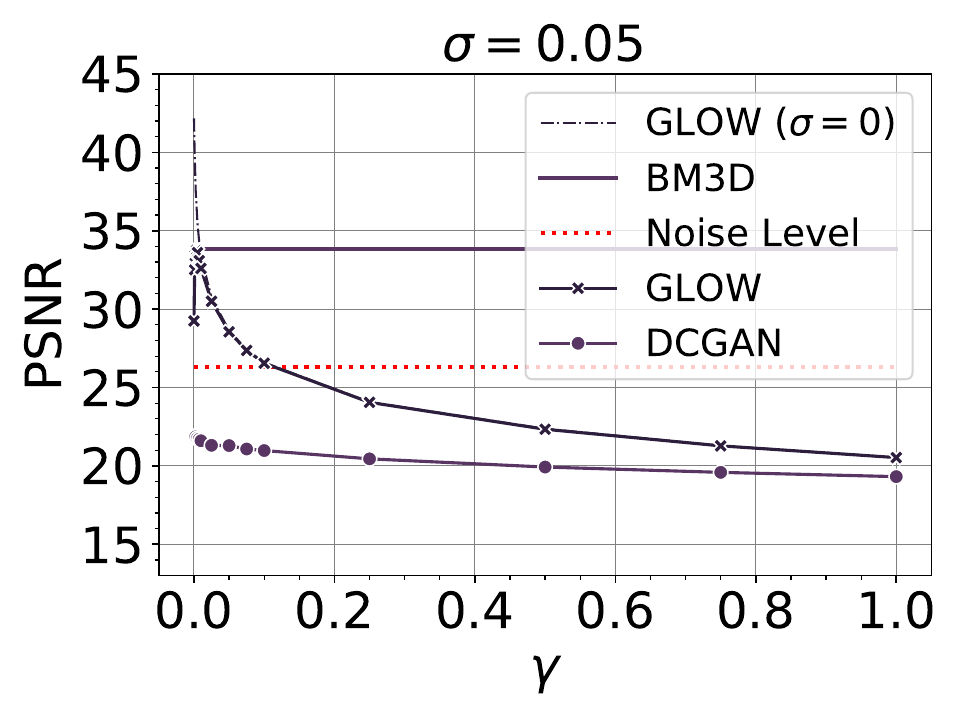}\\
    \includegraphics[width=0.4\textwidth]{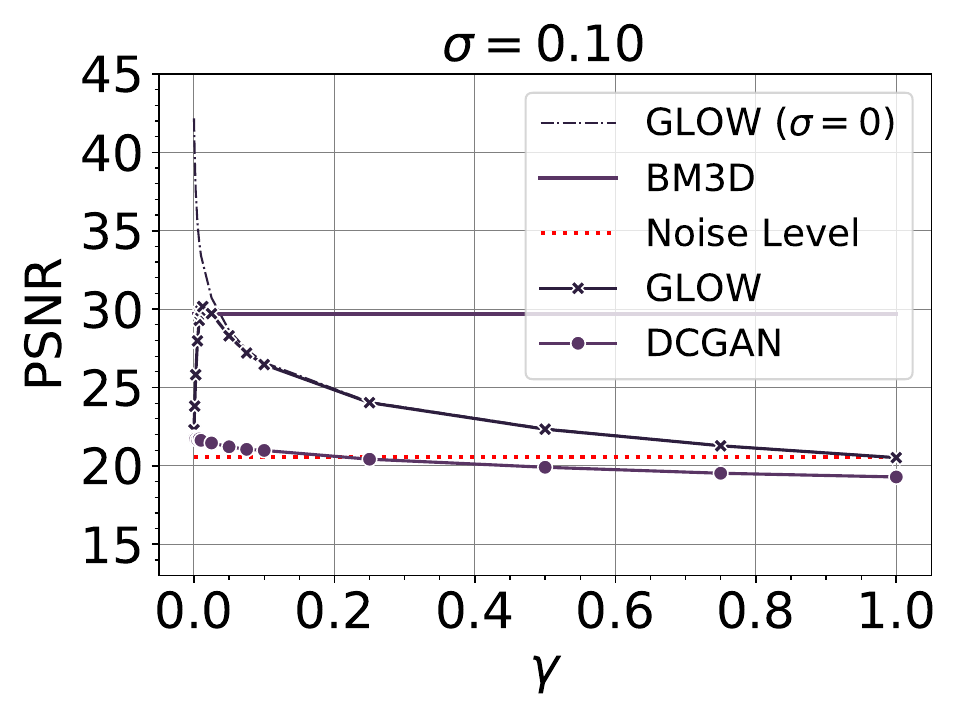}
    \includegraphics[width=0.4\textwidth]{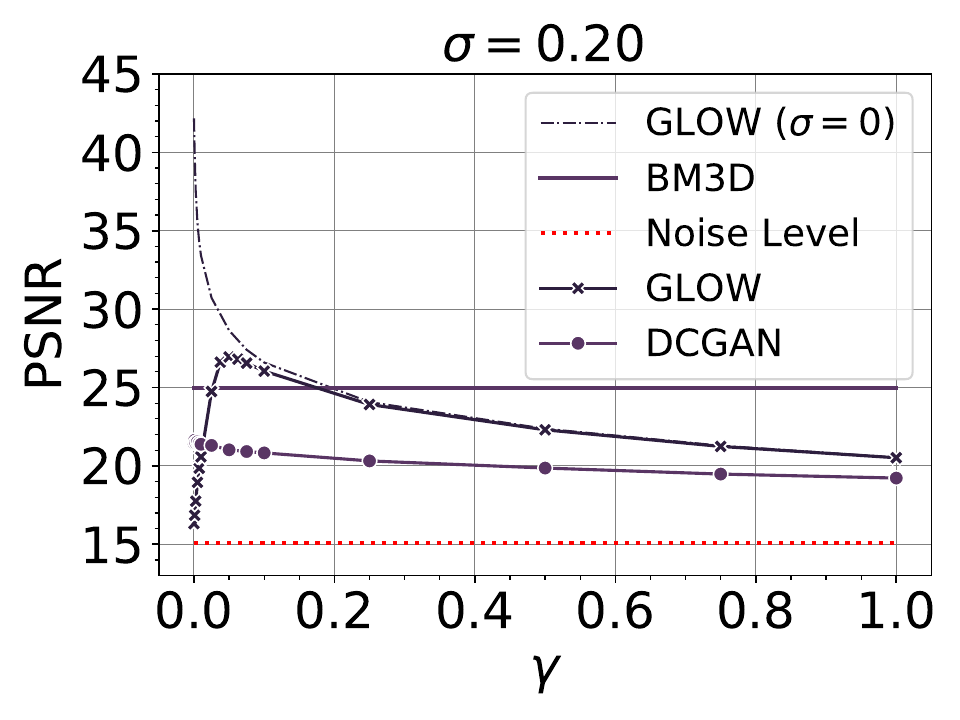}
    \caption{Image Denoising --- Recovered PSNR values as a function of $\gamma$ on $N = 50$ in-distribution test set CelebA images. We report the average PSNR after applying BM3D, under the DCGAN prior, and under the Glow prior, at noise levels $\sigma = 0.01, 0.05, 0.10, 0.20$. For reference, we also show the average PSNR of the original noisy images and the average PSNR of images recovered by the Glow prior in the noiseless case ($\sigma = 0$). }
    \label{fig:denosing-psnr-vs-gamma-all-levels}
\end{figure}

\begin{figure}[h!]
    \includegraphics[width=\textwidth]{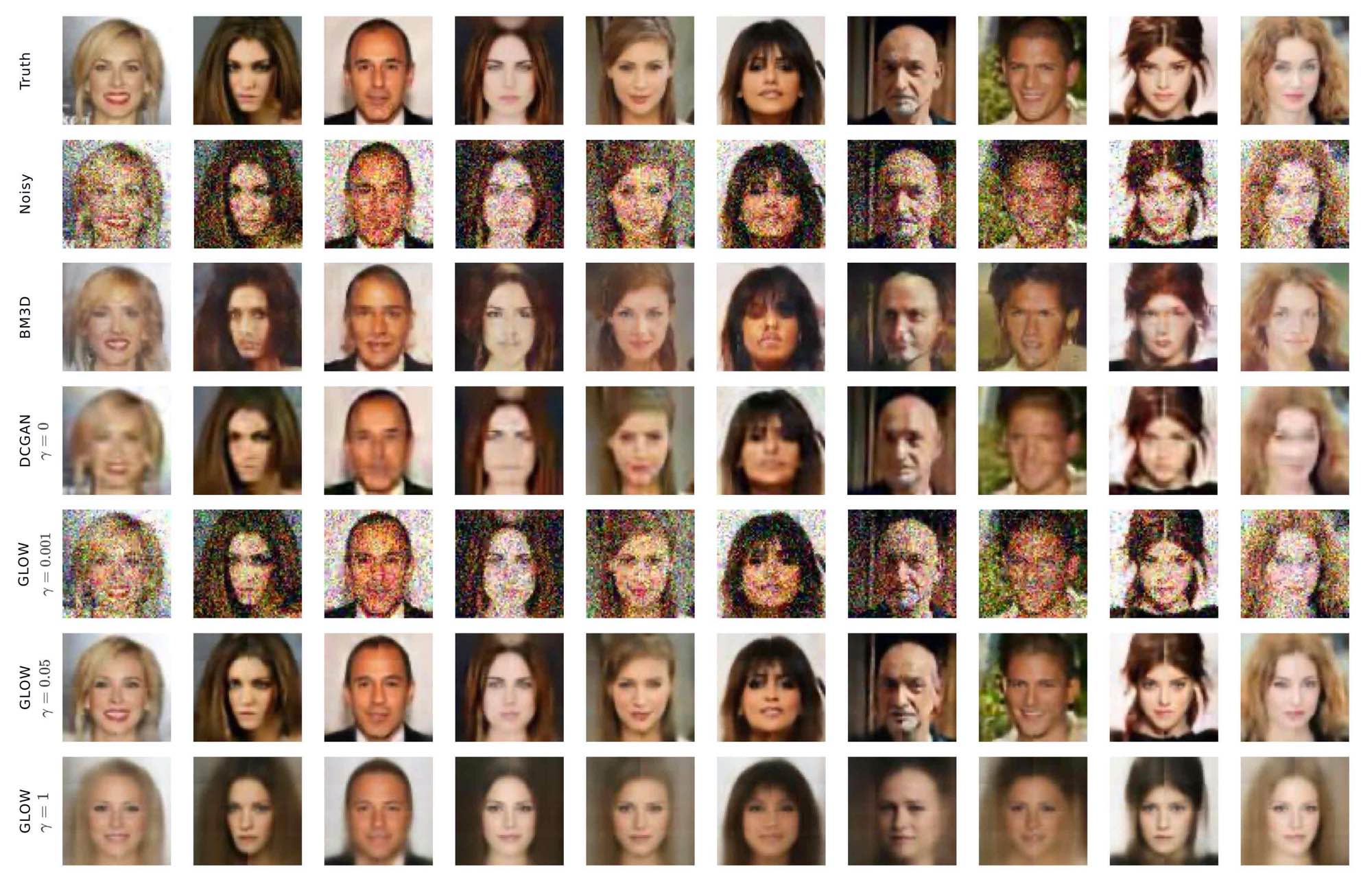}
    \caption{Image Denoising --- Visual comparisons under the Glow prior, the DCGAN prior, and BM3D at noise level $\sigma = 0.2$ on a sample of CelebA in-distribution test set images.  Under the DCGAN prior, we only show the case of $\gamma=0$ as this consistently gives the best performance.  Under the Glow prior, the best performance is achieved with $\gamma = 0.05$, overfitting of the image occurs with $\gamma=0.001$ and underfitting occurs with $\gamma = 1$. Note that the Glow prior with $\gamma = 0.05$ also gives a sharper image than BM3D.}
    \label{fig:denoising-vs-gamma-visually-higher-noise-2}
\end{figure}

\clearpage

\newpage

\subsection{Compressive Sensing: Additional Experiments}
We first analyze our formulation of compressive sensing as an empirical risk minimization problem. Recall that we solve the following optimization problem:
\begin{align*}
    \min_{z \in \R^n} \|A G(z) - y \|^2 & \qquad y = Ax_0 + \eta
\end{align*}
Where $A$ is a Gaussian random measurement matrix with i.i.d. $\mathcal{N}(0, 1/m)$ entries, and $\eta$ is noise normalized so that $\sqrt{\mathbb{E}\|\eta\|^2} = 0.1$. This formulation includes no explicit penalization on the likelihood of the latent code $z$, and instead relies on implicit regularization through the use of a gradient based optimization method initialized at $z_0 = 0$. To justify this formulation, we study various alternative methods of initialization: $z_0 = 0$, $z_0 \sim \mathcal{N}(0, 0.1^2I_n)$, $z_0 \sim \mathcal{N}(0, 0.7^2I_n)$, $z_0 = G^{-1}(x_0)$ with $x_0$ given by the solution to Lasso with respect to the wavelet basis, and $z_0 = G^{-1}(x_0)$ where $x_0$ is perturbed by a random point in the null space of $A$. For each initialization, we plot recovery performance on a compressive sensing task as a function of $\gamma$ for the regularized objective:
\begin{equation*}
    \min_{z \in \R^n} \|A G(z) - y \|^2 + \gamma \|z\|^2 \label{eqn:cs-penalty}
\end{equation*}
As shown in Fig. \ref{fig:cs-ablation}, the Glow model shows best performance with $\gamma=0$ with initialization at the zero vector, despite there being no explicit penalization on the likelihood of recovered latent codes. 
\begin{figure}[h]
    \centering
    \includegraphics[width=0.4\textwidth]{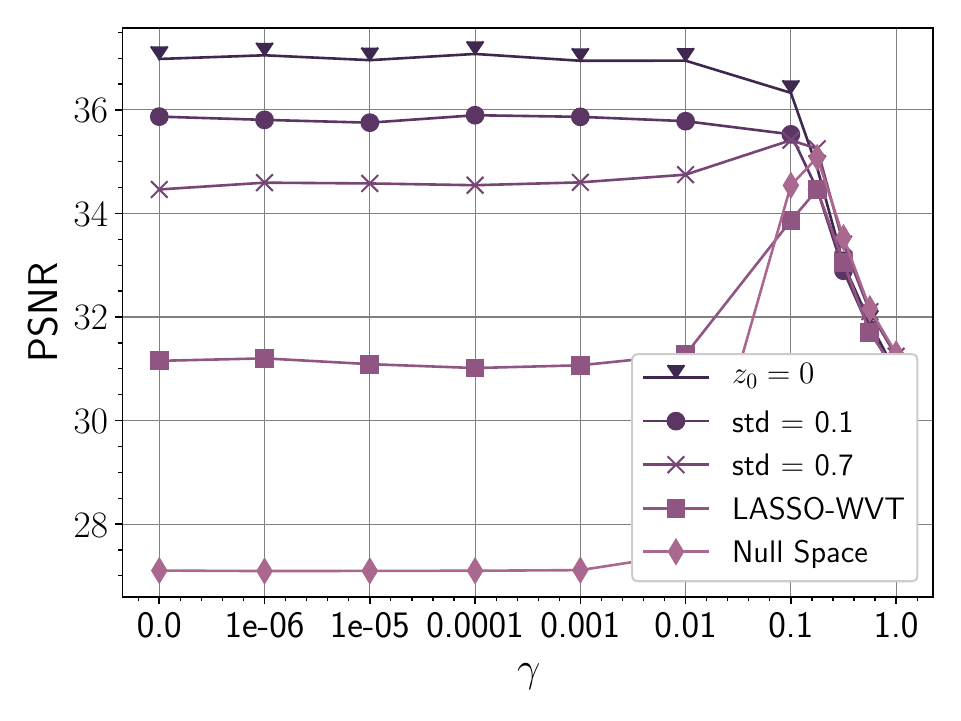}
    \caption{Performance of the Glow model in Compressive Sensing under various initialization strategies (as described in the text). The Glow model shows strongest performance when initialized with the zero vector and with no explicit penalization on the latent code likelihood. The task is $64$px CelebA image recovery using $m = 5000$ ($\approx 50$\%) Gaussian measurements.}
    \label{fig:cs-ablation}
\end{figure}

\newpage 
We show in Fig. \ref{fig:ssim} the results of our compressive sensing recovery experiments using the Structural Similarity index (SSIM) as a recovery quality metric \cite{wang2004ssim}. SSIM is designed to indicate perceptual quality by extracting structural information from images.

\begin{figure}[h]
    \centering
    \includegraphics[width=\textwidth]{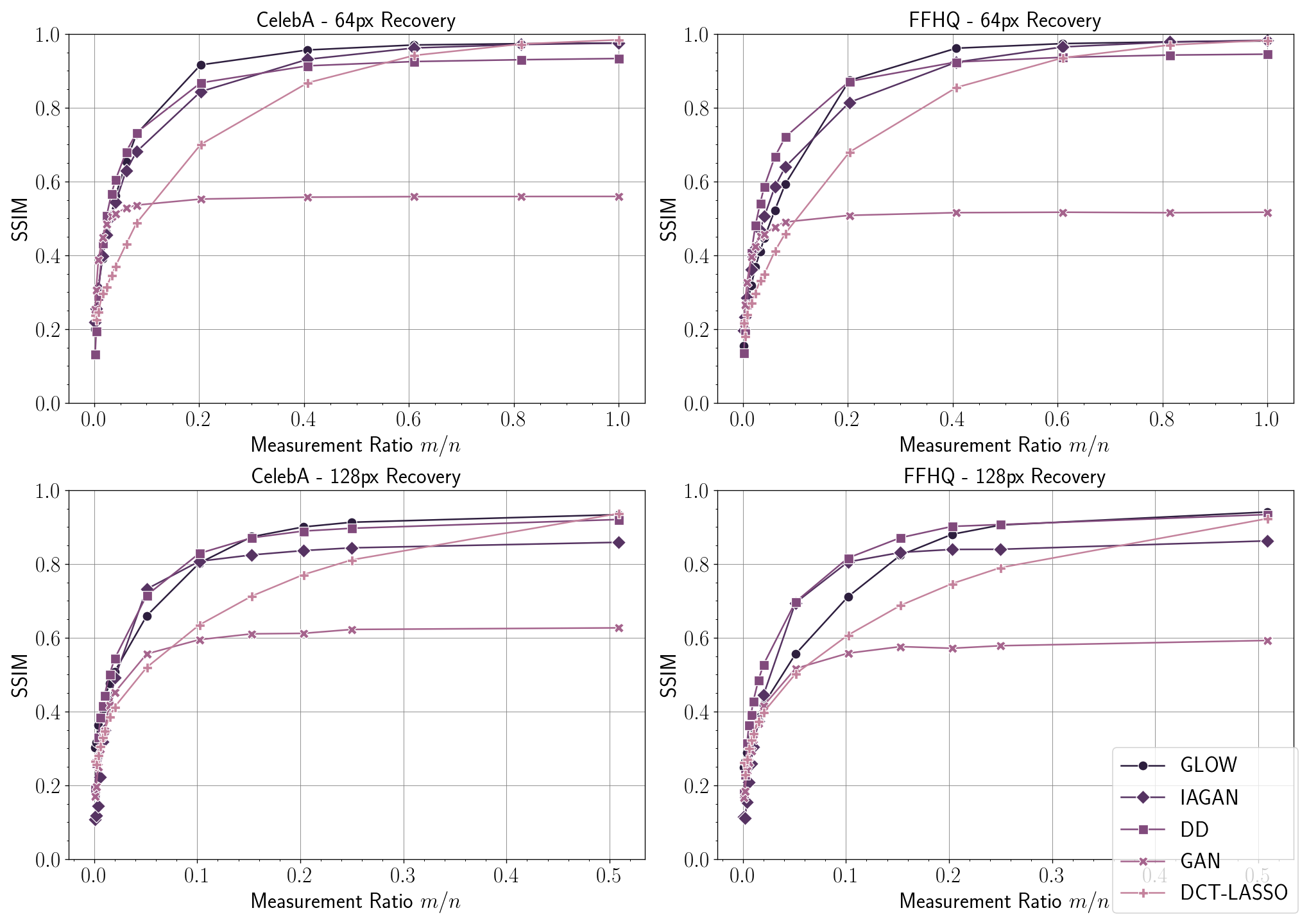}
    \caption{$64$px and $128$px Compressive Sensing results using $N=1000$ and $N=100$ test set images respectively. Each plot shows SSIM (higher is better, ranging in $[0, 1]$) across a variety of measurement ratios.}
    \label{fig:ssim}
\end{figure}

\clearpage \newpage 
\subsubsection{Compressive Sensing $64$px Sample Sheets}
We provide here additional sample sheets for our $64$px experiments in compressive sensing. For in-distribution images, we show in Figures \ref{fig:cs64_indist_2500} and \ref{fig:cs64_indist_10000} qualitative examples of the image reconstructions for the Glow prior, the DCGAN prior, the IA-DCGAN prior, and the Deep Decoder. We replicate the same experiments for out-of-distribution FFHQ images in Figures \ref{fig:cs64_outdist_2500} and \ref{fig:cs64_outdist_10000}.

\begin{figure}[h!]
    \centering
    \includegraphics[width=\textwidth]{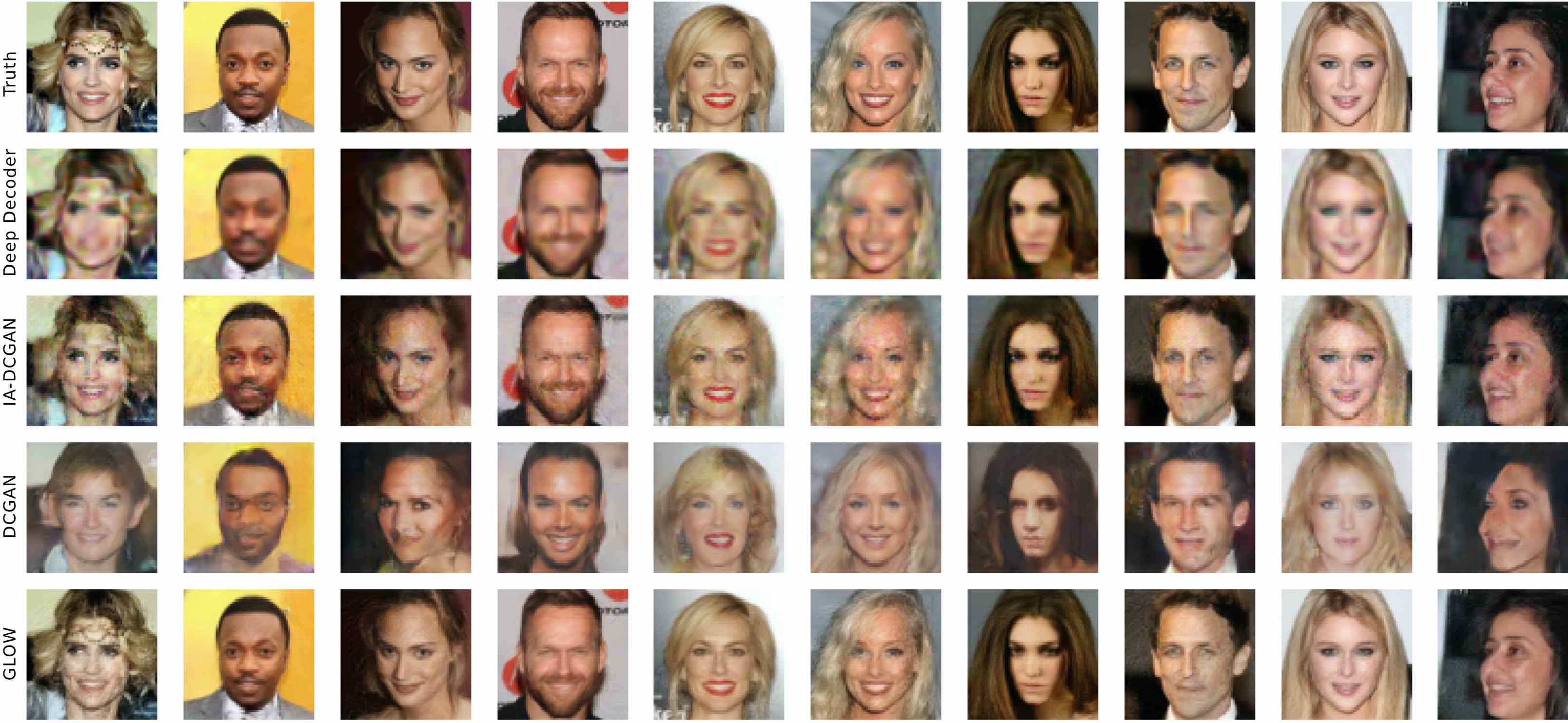}
    \caption{Compressive sensing visual comparisons --- Recoveries on a sample of in-distribution test set images with a number $m = 2500$ ($\approx 20\%$) of measurements under the Glow prior, the DCGAN prior, the IA-DCGAN prior, and the Deep Decoder.}
    \label{fig:cs64_indist_2500}
\end{figure}

\begin{figure}[h!]
    \centering
    \includegraphics[width=\textwidth]{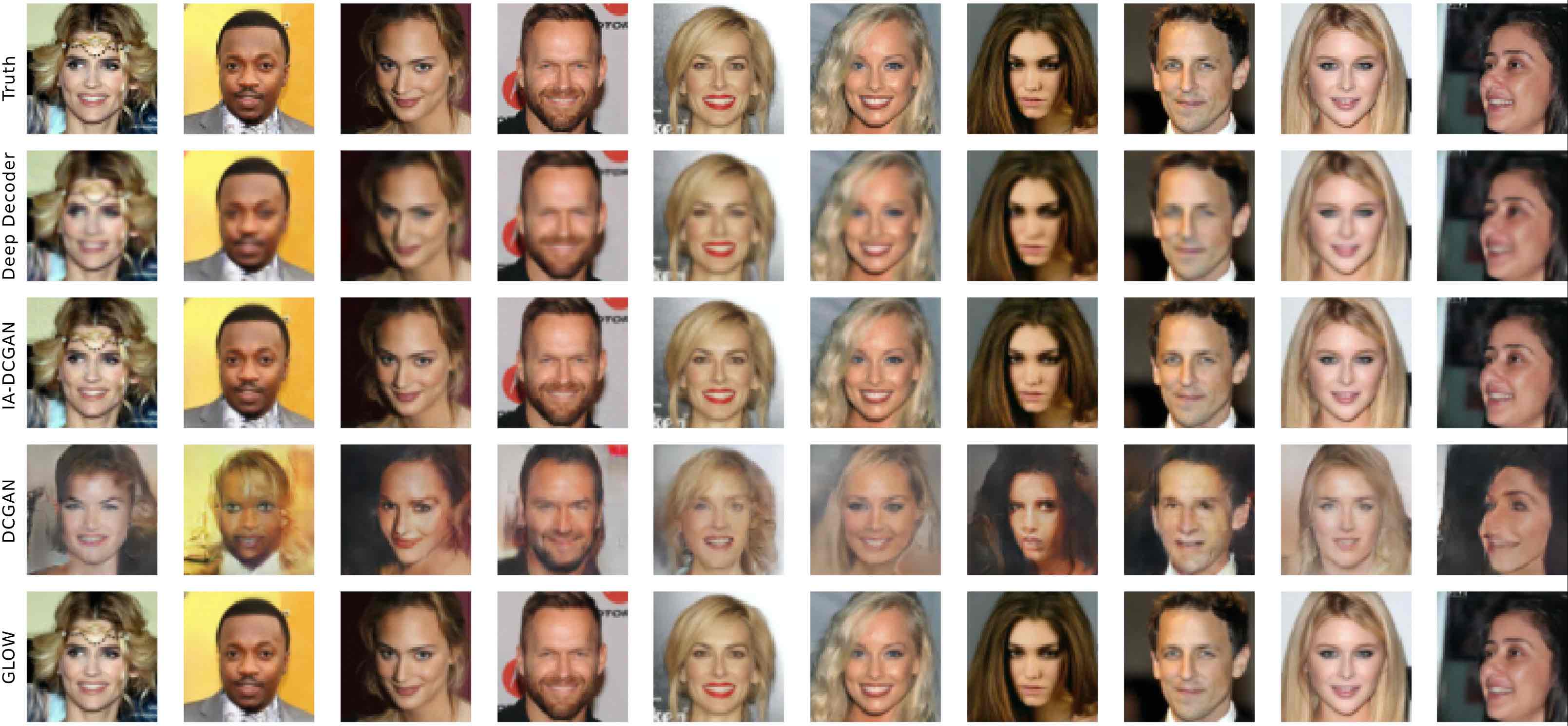}
    \caption{Compressive sensing visual comparisons --- Recoveries on a sample of in-distribution test set images with a number $m = 10000$ ($\approx 80\%$) of measurements under the Glow prior, the DCGAN prior, the IA-DCGAN prior, and the Deep Decoder.}
    \label{fig:cs64_indist_10000}
\end{figure}

\begin{figure}[h!]
    \centering
    \includegraphics[width=\textwidth]{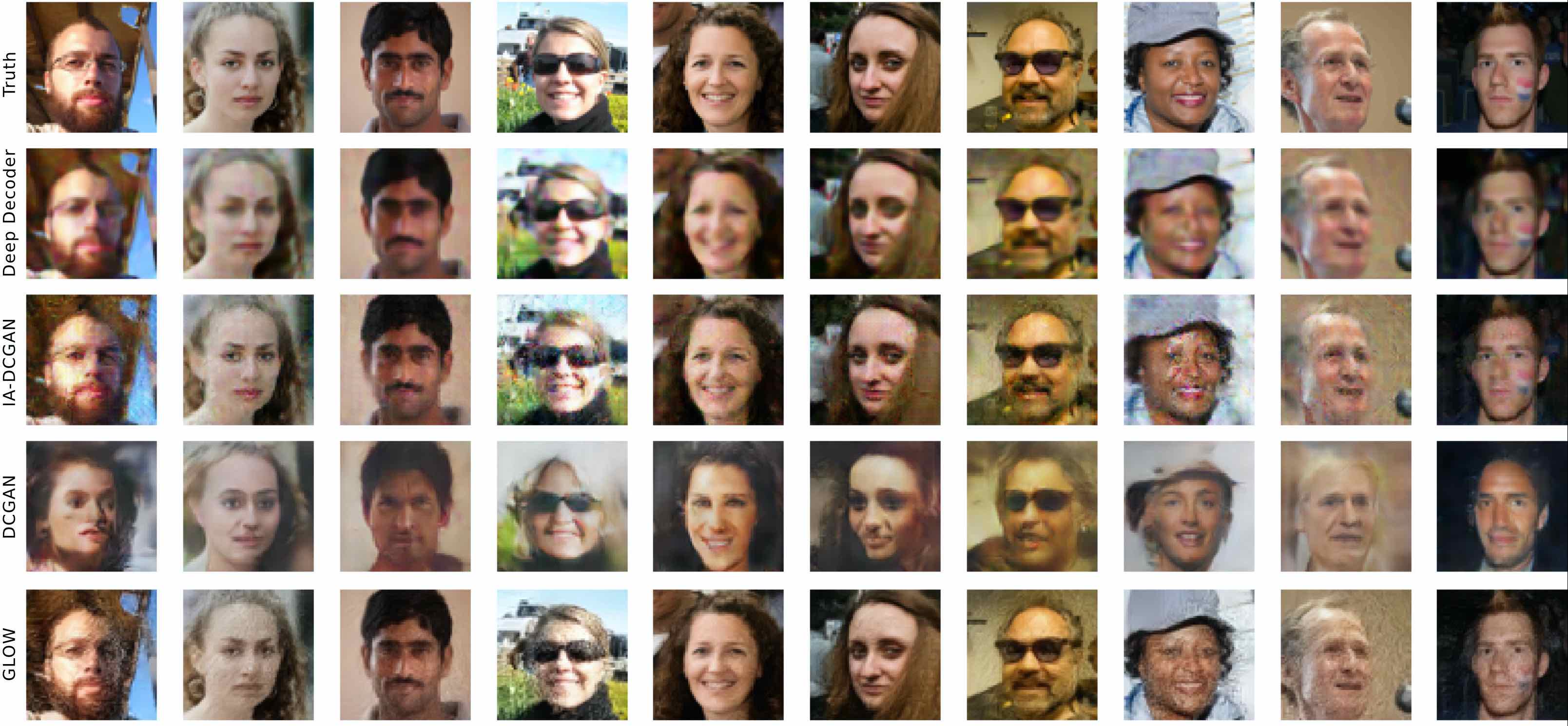}
    \caption{Compressive sensing visual comparisons --- Recoveries on a sample of out-of-distribution (FFHQ) test set images with a number $m = 2500$ ($\approx 20\%$) of measurements under the Glow prior, the DCGAN prior, the IA-DCGAN prior, and the Deep Decoder.}
    \label{fig:cs64_outdist_2500}
\end{figure}

\begin{figure}[h!]
    \centering
    \includegraphics[width=\textwidth]{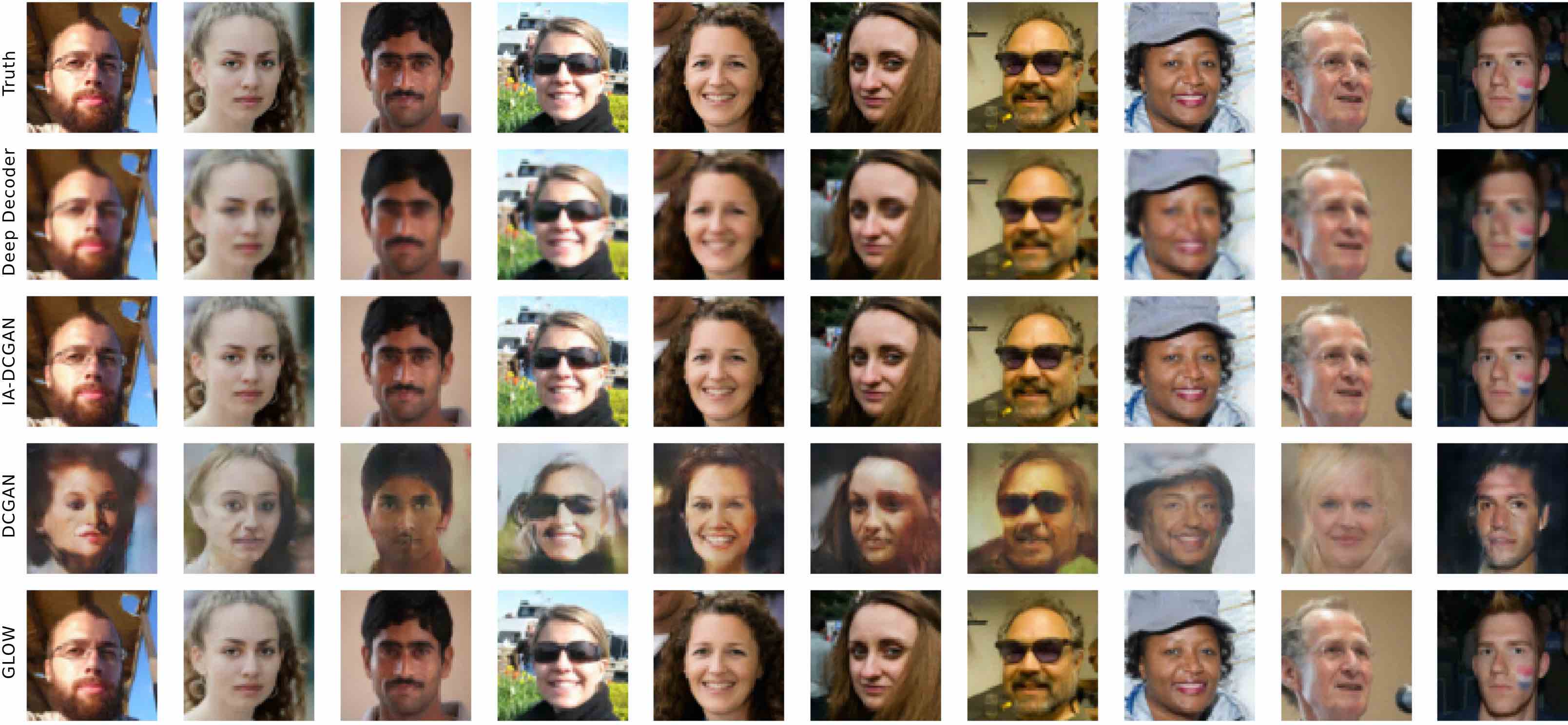}
    \caption{Compressive sensing visual comparisons --- Recoveries on a sample of out-of-distribution (FFHQ) test set images with a number $m = 10000$ ($\approx 80\%$) of measurements under the Glow prior, the DCGAN prior, the IA-DCGAN prior, and the Deep Decoder.}
    \label{fig:cs64_outdist_10000}
\end{figure}

\clearpage

To further investigate whether the Glow prior continues to be an effective proxy for arbitrarily out-of-distribution images, we tested arbitrary natural images such as a car, house door, and butterfly wings, which are all semantically unrelated to CelebA images. In general, we found that Glow is an effective prior at compressive sensing of out-of-distribution natural images, which are assigned a relatively high likelihood score (small normed latent representations) as compared to noisy, unnatural images. On these images, Glow also outperforms LASSO. This means that invertible networks have at least partially learned something more general about natural images from the CelebA dataset -- there may be some high level features that face images share with other natural images, such as smooth regions followed by discontinuities, etc. This allows the Glow model to extend its effectiveness as a prior to other natural images beyond just the training set.

As compared to in-distribution training images, however, semantically unrelated images are assigned very low-likelihood scores by the Glow model, causing instability issues. In particular, an L-BFGS search for the solution of an inverse problem to recover a low-likelihood image leads the iterates into neighborhoods of low-likelihood representations that may induce instability. All the network parameters such as scaling in the coupling layers of Glow network are learned to behave stably with high likelihood representations. However, on very low-likelihood representations, unseen during the training process, the networks becomes unstable and outputs of network begin to diverge to very large values; this may be due to several reasons, such as normalization (scaling) layers not being tuned to the unseen representations. See Section \ref{sec:arch} for details on our approach to handing these instabilities.

We show in Figure \ref{fig:OOD-2500} a comparison of the performance of the LASSO-DCT, LASSO-WVT, DCGAN, and Glow priors on the compressive sensing of $64$px out-of-distribution images for $m=2500$ ($\approx 20$\%) measurements.

\clearpage
\newpage
\begin{figure}[h!]
    \centering
    \raisebox{0.01in}{\rotatebox[origin=t]{0}{$m=2,500$}}\\
    \raisebox{0.15in}{\rotatebox[origin=t]{90}{\scriptsize Truth}}
    \includegraphics[width=0.069\textwidth]{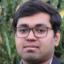}
    \includegraphics[width=0.069\textwidth]{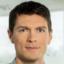}
    \includegraphics[width=0.069\textwidth]{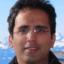}
    \includegraphics[width=0.069\textwidth]{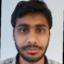}
    \includegraphics[width=0.069\textwidth]{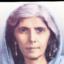}
    \includegraphics[width=0.069\textwidth]{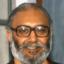}
    \includegraphics[width=0.069\textwidth]{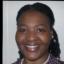}
    \includegraphics[width=0.069\textwidth]{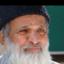}
    \includegraphics[width=0.069\textwidth]{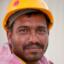}
    \includegraphics[width=0.069\textwidth]{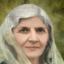}
    \includegraphics[width=0.069\textwidth]{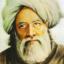}
    \includegraphics[width=0.069\textwidth]{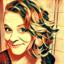}
    \includegraphics[width=0.069\textwidth]{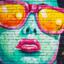}\\
    \raisebox{0.17in}{\rotatebox[origin=t]{90}{\scriptsize DCT}}
    \includegraphics[width=0.069\textwidth]{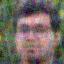}
    \includegraphics[width=0.069\textwidth]{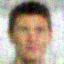}
    \includegraphics[width=0.069\textwidth]{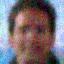}
    \includegraphics[width=0.069\textwidth]{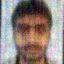}
    \includegraphics[width=0.069\textwidth]{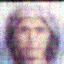}
    \includegraphics[width=0.069\textwidth]{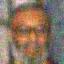}
    \includegraphics[width=0.069\textwidth]{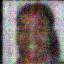}
    \includegraphics[width=0.069\textwidth]{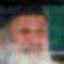}
    \includegraphics[width=0.069\textwidth]{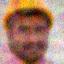}
    \includegraphics[width=0.069\textwidth]{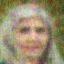}
    \includegraphics[width=0.069\textwidth]{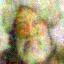}
    \includegraphics[width=0.069\textwidth]{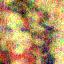}
    \includegraphics[width=0.069\textwidth]{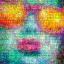}\\
    \raisebox{0.17in}{\rotatebox[origin=t]{90}{\scriptsize WVT}}
    \includegraphics[width=0.069\textwidth]{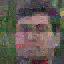}
    \includegraphics[width=0.069\textwidth]{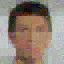}
    \includegraphics[width=0.069\textwidth]{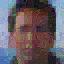}
    \includegraphics[width=0.069\textwidth]{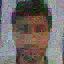}
    \includegraphics[width=0.069\textwidth]{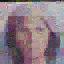}
    \includegraphics[width=0.069\textwidth]{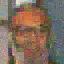}
    \includegraphics[width=0.069\textwidth]{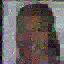}
    \includegraphics[width=0.069\textwidth]{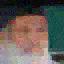}
    \includegraphics[width=0.069\textwidth]{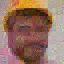}
    \includegraphics[width=0.069\textwidth]{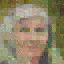}
    \includegraphics[width=0.069\textwidth]{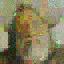}
    \includegraphics[width=0.069\textwidth]{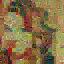}
    \includegraphics[width=0.069\textwidth]{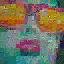}\\
    \raisebox{0.17in}{\rotatebox[origin=t]{90}{\scriptsize DCGAN}}
    \includegraphics[width=0.069\textwidth]{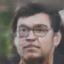}
    \includegraphics[width=0.069\textwidth]{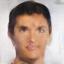}
    \includegraphics[width=0.069\textwidth]{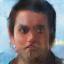}
    \includegraphics[width=0.069\textwidth]{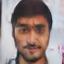}
    \includegraphics[width=0.069\textwidth]{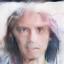}
    \includegraphics[width=0.069\textwidth]{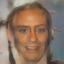}
    \includegraphics[width=0.069\textwidth]{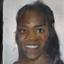}
    \includegraphics[width=0.069\textwidth]{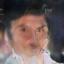}
    \includegraphics[width=0.069\textwidth]{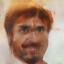}
    \includegraphics[width=0.069\textwidth]{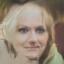}
    \includegraphics[width=0.069\textwidth]{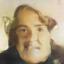}
    \includegraphics[width=0.069\textwidth]{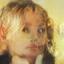}
    \includegraphics[width=0.069\textwidth]{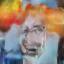}\\
    \raisebox{0.17in}{\rotatebox[origin=t]{90}{\scriptsize Glow}}
    \includegraphics[width=0.069\textwidth]{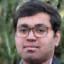}
    \includegraphics[width=0.069\textwidth]{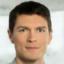}
    \includegraphics[width=0.069\textwidth]{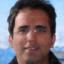}
    \includegraphics[width=0.069\textwidth]{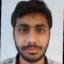}
    \includegraphics[width=0.069\textwidth]{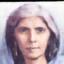}
    \includegraphics[width=0.069\textwidth]{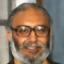}
    \includegraphics[width=0.069\textwidth]{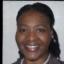}
    \includegraphics[width=0.069\textwidth]{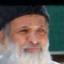}
    \includegraphics[width=0.069\textwidth]{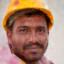}
    \includegraphics[width=0.069\textwidth]{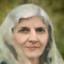}
    \includegraphics[width=0.069\textwidth]{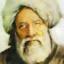}
    \includegraphics[width=0.069\textwidth]{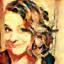}
    \includegraphics[width=0.069\textwidth]{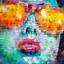}\\
    \raisebox{0.15in}{\rotatebox[origin=t]{90}{\scriptsize Truth}}
    \includegraphics[width=0.069\textwidth]{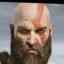}
    \includegraphics[width=0.069\textwidth]{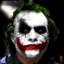}
    \includegraphics[width=0.069\textwidth]{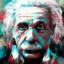}
    \includegraphics[width=0.069\textwidth]{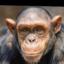}
    \includegraphics[width=0.069\textwidth]{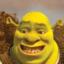}
    \includegraphics[width=0.069\textwidth]{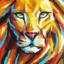}
    \includegraphics[width=0.069\textwidth]{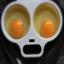}
    \includegraphics[width=0.069\textwidth]{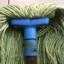}
    \includegraphics[width=0.069\textwidth]{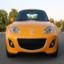}
    \includegraphics[width=0.069\textwidth]{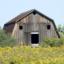}
    \includegraphics[width=0.069\textwidth]{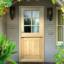}
    \includegraphics[width=0.069\textwidth]{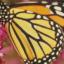}
    \includegraphics[width=0.069\textwidth]{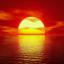}\\
    \raisebox{0.17in}{\rotatebox[origin=t]{90}{\scriptsize DCT}}
    \includegraphics[width=0.069\textwidth]{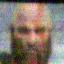}
    \includegraphics[width=0.069\textwidth]{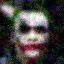}
    \includegraphics[width=0.069\textwidth]{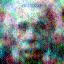}
    \includegraphics[width=0.069\textwidth]{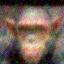}
    \includegraphics[width=0.069\textwidth]{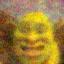}
    \includegraphics[width=0.069\textwidth]{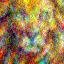}
    \includegraphics[width=0.069\textwidth]{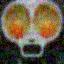}
    \includegraphics[width=0.069\textwidth]{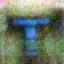}
    \includegraphics[width=0.069\textwidth]{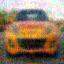}
    \includegraphics[width=0.069\textwidth]{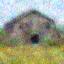}
    \includegraphics[width=0.069\textwidth]{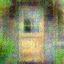}
    \includegraphics[width=0.069\textwidth]{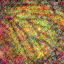}
    \includegraphics[width=0.069\textwidth]{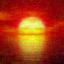}\\
    \raisebox{0.17in}{\rotatebox[origin=t]{90}{\scriptsize WVT}}
    \includegraphics[width=0.069\textwidth]{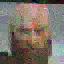}
    \includegraphics[width=0.069\textwidth]{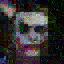}
    \includegraphics[width=0.069\textwidth]{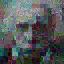}
    \includegraphics[width=0.069\textwidth]{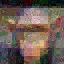}
    \includegraphics[width=0.069\textwidth]{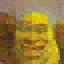}
    \includegraphics[width=0.069\textwidth]{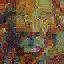}
\includegraphics[width=0.069\textwidth]{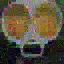}
\includegraphics[width=0.069\textwidth]{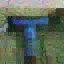}
\includegraphics[width=0.069\textwidth]{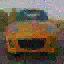}
\includegraphics[width=0.069\textwidth]{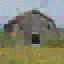}
\includegraphics[width=0.069\textwidth]{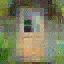}
\includegraphics[width=0.069\textwidth]{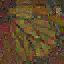}
\includegraphics[width=0.069\textwidth]{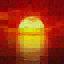}\\
    \raisebox{0.17in}{\rotatebox[origin=t]{90}{\scriptsize DCGAN}}
    \includegraphics[width=0.069\textwidth]{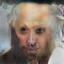}
    \includegraphics[width=0.069\textwidth]{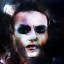}
    \includegraphics[width=0.069\textwidth]{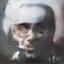}
    \includegraphics[width=0.069\textwidth]{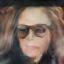}
    \includegraphics[width=0.069\textwidth]{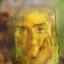}
    \includegraphics[width=0.069\textwidth]{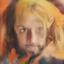}
    \includegraphics[width=0.069\textwidth]{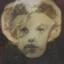}
    \includegraphics[width=0.069\textwidth]{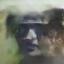}
    \includegraphics[width=0.069\textwidth]{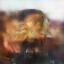}
    \includegraphics[width=0.069\textwidth]{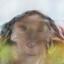}
    \includegraphics[width=0.069\textwidth]{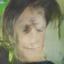}
    \includegraphics[width=0.069\textwidth]{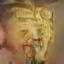}
    \includegraphics[width=0.069\textwidth]{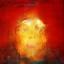}\\
    \raisebox{0.17in}{\rotatebox[origin=t]{90}{\scriptsize Glow}}
    \includegraphics[width=0.069\textwidth]{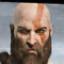}
    \includegraphics[width=0.069\textwidth]{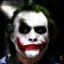}
    \includegraphics[width=0.069\textwidth]{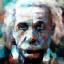}
    \includegraphics[width=0.069\textwidth]{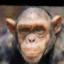}
    \includegraphics[width=0.069\textwidth]{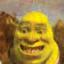}
    \includegraphics[width=0.069\textwidth]{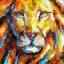}
    \includegraphics[width=0.069\textwidth]{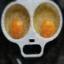}
    \includegraphics[width=0.069\textwidth]{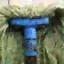}
    \includegraphics[width=0.069\textwidth]{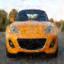}
    \includegraphics[width=0.069\textwidth]{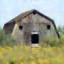}
    \includegraphics[width=0.069\textwidth]{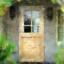}
    \includegraphics[width=0.069\textwidth]{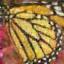}
    \includegraphics[width=0.069\textwidth]{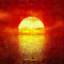}
    \caption{Compressive sensing of $64$px out-of-distribution images given $m = 2500$ ($\approx 20$\%) measurements at a noise level $\sqrt{\E \|\eta\|^2} = 0.1$. We provide a visual comparison of recoveries of the LASSO-DCT, LASSO-WVT, DCGAN, and Glow priors, where the DCGAN and Glow priors are trained on CelebA images. In each case, we choose values of the penalization parameter $\gamma$ to yield the best performance. We use $\gamma = 0$ for both the DCGAN and Glow priors, and optimize $\gamma$ for each recovery using the LASSO-WVT and LASSO-DCT priors. }
    \label{fig:OOD-2500}
\end{figure}

\newpage
\subsubsection{Compressive Sensing $128$px Sample Sheets}
We provide here additional sample sheets for our $128$px experiments in compressive sensing. For in-distribution images, we show in Figures \ref{fig:cs_indist_10000} and \ref{fig:cs_indist_2500} qualitative examples of the image reconstructions for the Glow prior, the PGGAN prior, the IA-PGGAN prior, and the Deep Decoder. We replicate the same experiments for out-of-distribution (FFHQ) images in Figures \ref{fig:cs_outdist_10000} and \ref{fig:cs_outdist_2500}.


\begin{figure}[h!]
    \centering
    \includegraphics[width=\textwidth]{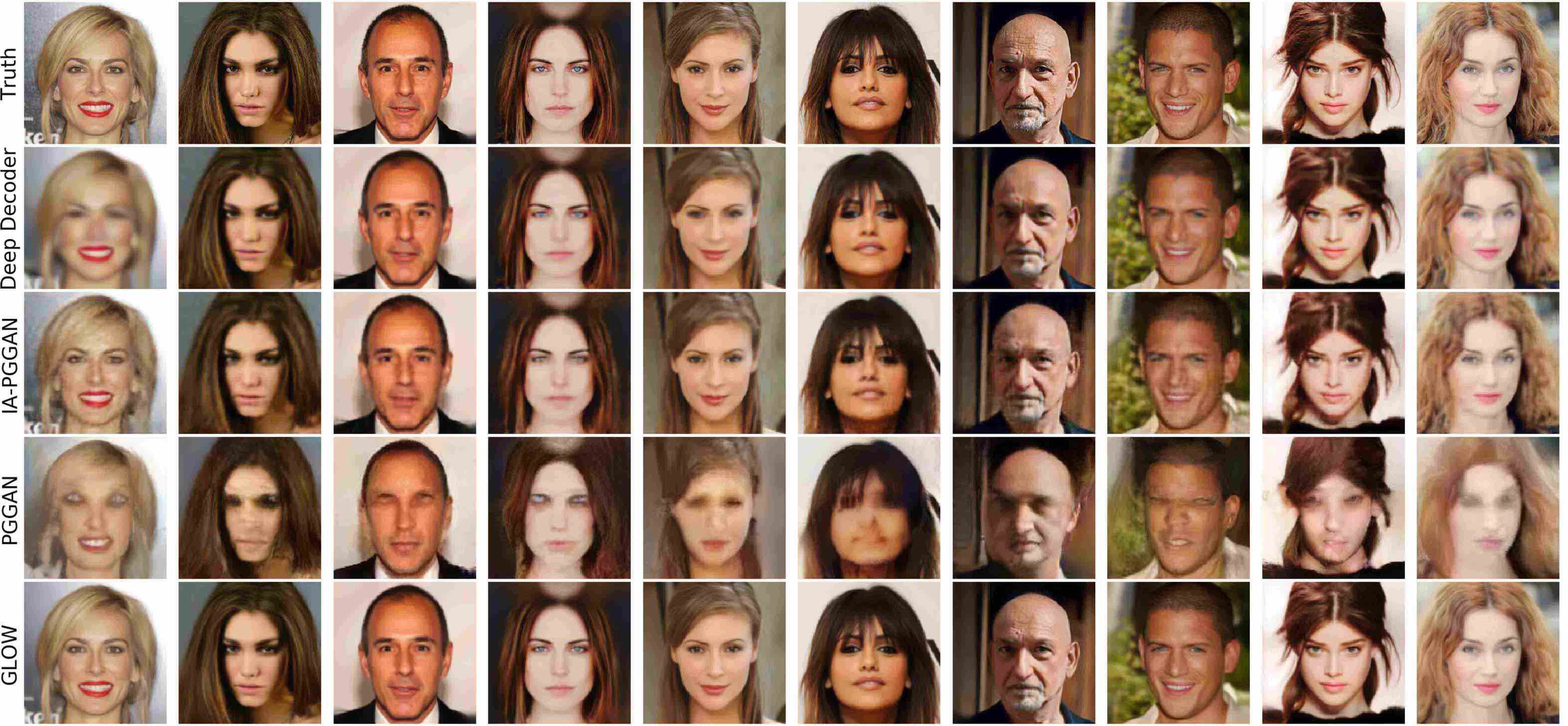}
    \caption{Compressive sensing visual comparisons --- Recoveries on a sample of in-distribution test set images with a number $m = 10000$ ($\approx 20\%$) of measurements under the Glow prior, the PGGAN prior, the IA-PGGAN prior, and the Deep Decoder.}
    \label{fig:cs_indist_10000}
\end{figure}

\begin{figure}[h!]
    \centering
    \includegraphics[width=\textwidth]{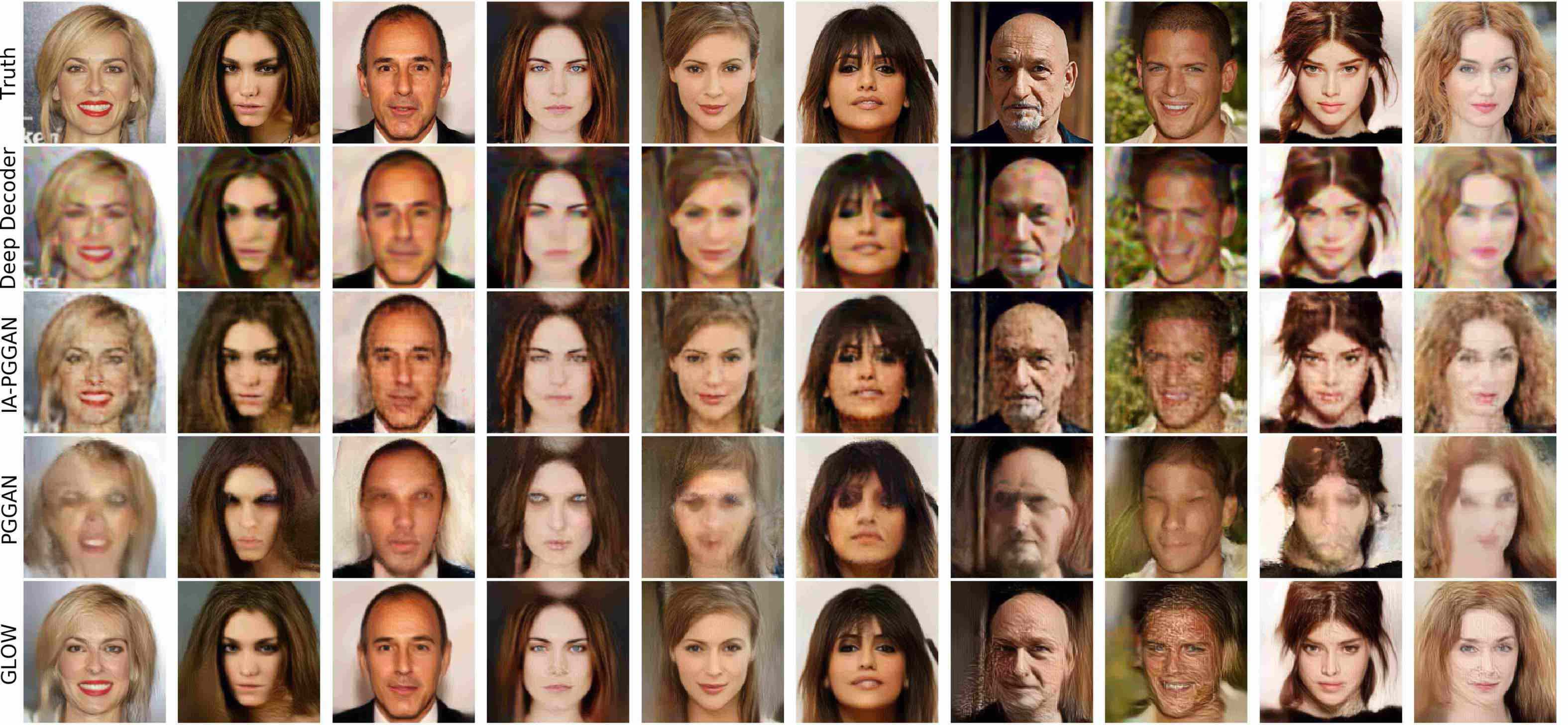}
    \caption{Compressive sensing visual comparisons --- Recoveries on a sample of in-distribution test set images with a number $m = 2500$ ($\approx 5\%$) of measurements under the Glow prior, the PGGAN prior, the IA-PGGAN prior, and the Deep Decoder.}
    \label{fig:cs_indist_2500}
\end{figure}

\begin{figure}[h!]
    \centering
    \includegraphics[width=\textwidth]{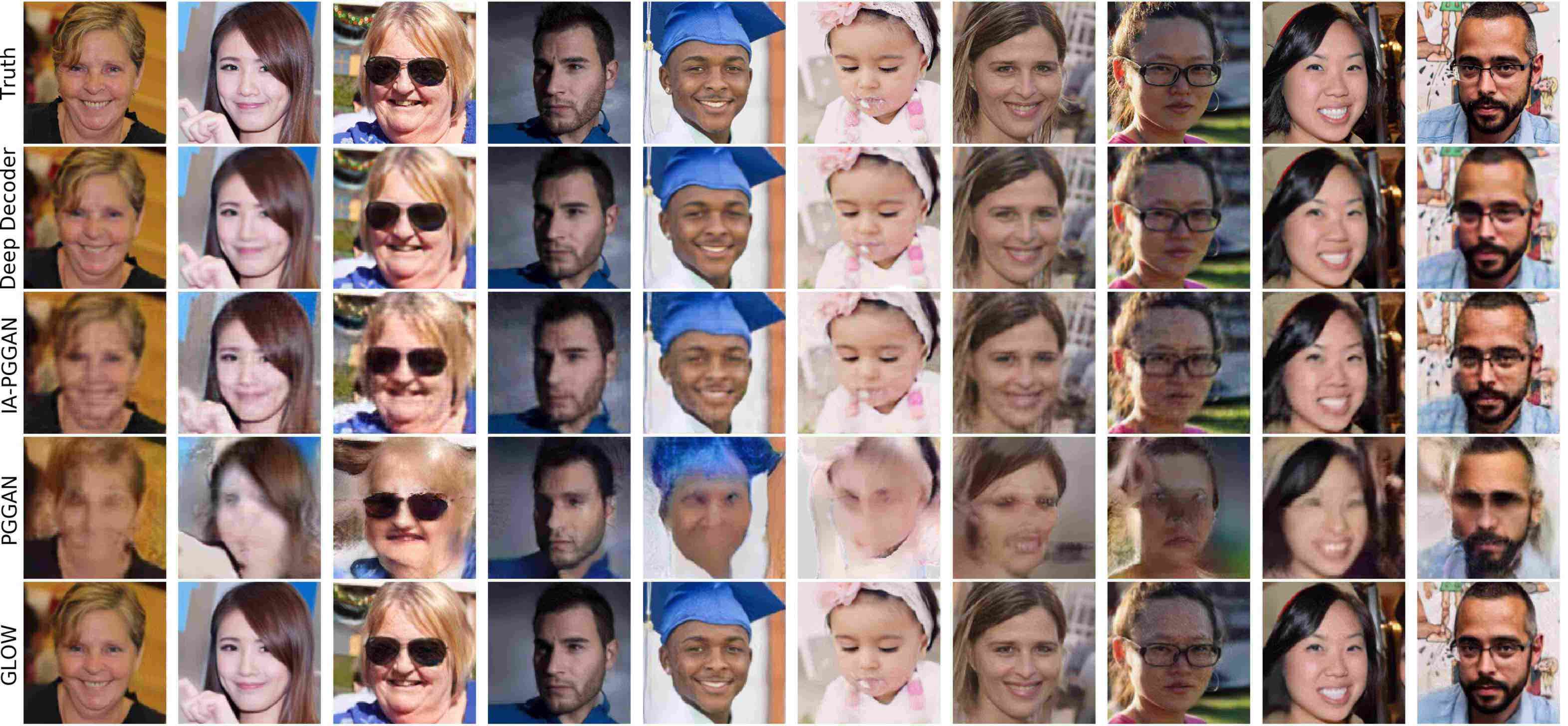}
    \caption{Compressive sensing visual comparisons --- Recoveries on a sample of out-of-distribution (FFHQ) test set images with a number $m = 10000$ ($\approx 20\%$) of measurements under the Glow prior, the PGGAN prior, the IA-PGGAN prior, and the Deep Decoder.}
    \label{fig:cs_outdist_10000}
\end{figure}

\begin{figure}[h!]
    \centering
    \includegraphics[width=\textwidth]{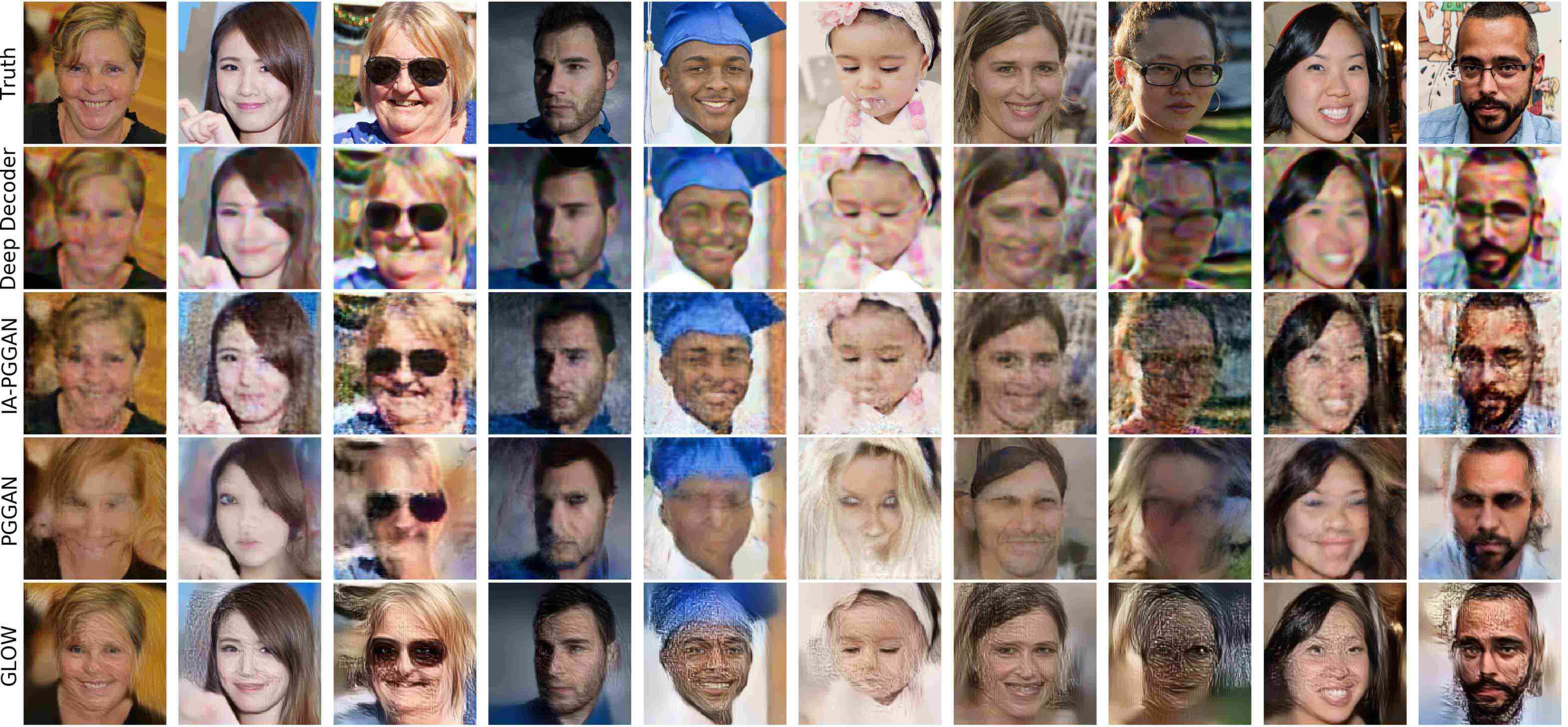}
    \caption{Compressive sensing visual comparisons --- Recoveries on a sample of out-of-distribution (FFHQ) test set images with a number $m = 2500$ ($\approx 5\%$) of measurements under the Glow prior, the PGGAN prior, the IA-PGGAN prior, and the Deep Decoder.}
    \label{fig:cs_outdist_2500}
\end{figure}

\clearpage

\newpage
\subsection{Inpainting}

We present here qualitative results on image inpainting under the DCGAN prior and the Glow prior on the CelebA test set. Compared to DCGAN, the reconstructions from Glow are of noticeably higher visual quality. 
\begin{figure}[h!]
    \centering
    \raisebox{0.2in}{\rotatebox[origin=t]{90}{\scriptsize Truth}}
    \raisebox{0.2in}{\rotatebox[origin=t]{90}{\scriptsize n.a.}}
    \includegraphics[width=0.072\textwidth]{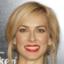}
    \includegraphics[width=0.072\textwidth]{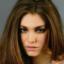}
    \includegraphics[width=0.072\textwidth]{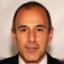}
    \includegraphics[width=0.072\textwidth]{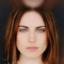}
    \includegraphics[width=0.072\textwidth]{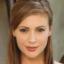}
    \includegraphics[width=0.072\textwidth]{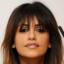}
    \includegraphics[width=0.072\textwidth]{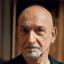}
    \includegraphics[width=0.072\textwidth]{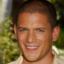}
    \includegraphics[width=0.072\textwidth]{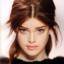}
    \includegraphics[width=0.072\textwidth]{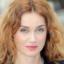}
    \includegraphics[width=0.072\textwidth]{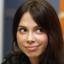}
    \includegraphics[width=0.072\textwidth]{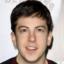}
    
    \raisebox{0.2in}{\rotatebox[origin=t]{90}{\scriptsize Masked}}
    \raisebox{0.2in}{\rotatebox[origin=t]{90}{\scriptsize n.a.}}
    \includegraphics[width=0.072\textwidth]{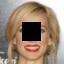}
    \includegraphics[width=0.072\textwidth]{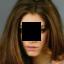}
    \includegraphics[width=0.072\textwidth]{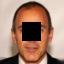}
    \includegraphics[width=0.072\textwidth]{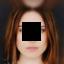}
    \includegraphics[width=0.072\textwidth]{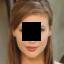}
    \includegraphics[width=0.072\textwidth]{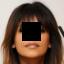}
    \includegraphics[width=0.072\textwidth]{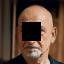}
    \includegraphics[width=0.072\textwidth]{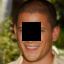}
    \includegraphics[width=0.072\textwidth]{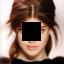}
    \includegraphics[width=0.072\textwidth]{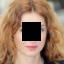}
    \includegraphics[width=0.072\textwidth]{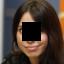}
    \includegraphics[width=0.072\textwidth]{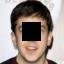}\\[-0.2em]
    
    \raisebox{0.2in}{\rotatebox[origin=t]{90}{\scriptsize DCGAN}}
    \raisebox{0.2in}{\rotatebox[origin=t]{90}{\scriptsize $\gamma = 0$}}\hspace{-0.3em}
    \includegraphics[width=0.072\textwidth]{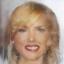}
    \includegraphics[width=0.072\textwidth]{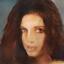}
    \includegraphics[width=0.072\textwidth]{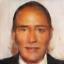}
    \includegraphics[width=0.072\textwidth]{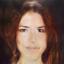}
    \includegraphics[width=0.072\textwidth]{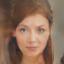}
    \includegraphics[width=0.072\textwidth]{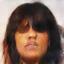}
    \includegraphics[width=0.072\textwidth]{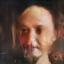}
    \includegraphics[width=0.072\textwidth]{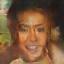}
    \includegraphics[width=0.072\textwidth]{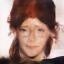}
    \includegraphics[width=0.072\textwidth]{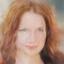}
    \includegraphics[width=0.072\textwidth]{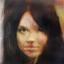}
    \includegraphics[width=0.072\textwidth]{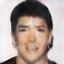}
    
    \raisebox{0.2in}{\rotatebox[origin=t]{90}{\scriptsize Glow}}
    \raisebox{0.2in}{\rotatebox[origin=t]{90}{\scriptsize $\gamma = 0$}}\hspace{-0.3em}
    \includegraphics[width=0.072\textwidth]{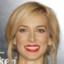}
    \includegraphics[width=0.072\textwidth]{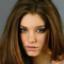}
    \includegraphics[width=0.072\textwidth]{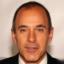}
    \includegraphics[width=0.072\textwidth]{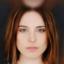}
    \includegraphics[width=0.072\textwidth]{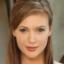}
    \includegraphics[width=0.072\textwidth]{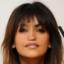}
    \includegraphics[width=0.072\textwidth]{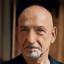}
    \includegraphics[width=0.072\textwidth]{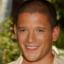}
    \includegraphics[width=0.072\textwidth]{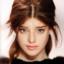}
    \includegraphics[width=0.072\textwidth]{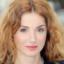}
    \includegraphics[width=0.072\textwidth]{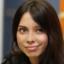}
    \includegraphics[width=0.072\textwidth]{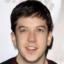}
    
    \caption{Image inpainiting results on a sample of the CelebA test set. Masked images are recovered using the DCGAN prior and the Glow prior. Recoveries under the DCGAN prior are skewed and blurred whereas the Glow prior leads to sharper and coherent inpainted images. For both Glow and DCGAN, we set $\gamma = 0$.}
\end{figure}

\newpage
\subsubsection{Image Inpainting on Out of Distribution Images}
We now perform image inpainiting under the Glow prior and the DCGAN prior, each trained on CelebA. Figure \ref{fig:inpainting-OOD-supp} shows the visuals of out-of-distribution inpainting.  As before, the DCGAN prior continues to suffer due to representation error and data bias while Glow achieves reasonable reconstructions on out-of-distribution images which are semantically similar to CelebA faces.  As one deviates to other natural images such as houses, doors, and butterfly wings,  the inpainting performance deteriorates. In compressive sensing, Glow performs much better on such arbitrarily out-of-distribution images as in this case, good recoveries only require the network only to assign a higher likelihood score to the true image compared to the all the candidate static images given by the null space of the measurement operator. 
\begin{figure}[h!]
    \centering
    \raisebox{0.25in}{\rotatebox[origin=t]{90}{\scriptsize Truth}}
    \includegraphics[width=0.092\textwidth]{./images/celeba/transfer_faces/original/001.jpg}
    \includegraphics[width=0.092\textwidth]{./images/celeba/transfer_faces/original/002.jpg}
    \includegraphics[width=0.092\textwidth]{./images/celeba/transfer_faces/original/004.jpg}
    \includegraphics[width=0.092\textwidth]{./images/celeba/transfer_faces/original/005.jpg}
    \includegraphics[width=0.092\textwidth]{./images/celeba/transfer_faces/original/006.jpg}
    \includegraphics[width=0.092\textwidth]{./images/celeba/transfer_faces/original/007.jpg}
    \includegraphics[width=0.092\textwidth]{./images/celeba/transfer_faces/original/014.jpg}
    \includegraphics[width=0.092\textwidth]{./images/celeba/transfer_faces/original/008.jpg}
    \includegraphics[width=0.092\textwidth]{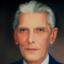}
    \includegraphics[width=0.092\textwidth]{./images/celeba/transfer_faces/original/011.jpg}\\
    \raisebox{0.25in}{\rotatebox[origin=t]{90}{\scriptsize Masked}}
    \includegraphics[width=0.092\textwidth]{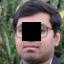}
    \includegraphics[width=0.092\textwidth]{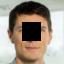}
    \includegraphics[width=0.092\textwidth]{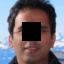}
    \includegraphics[width=0.092\textwidth]{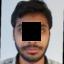}
    \includegraphics[width=0.092\textwidth]{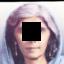}
    \includegraphics[width=0.092\textwidth]{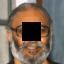}
    \includegraphics[width=0.092\textwidth]{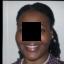}
    \includegraphics[width=0.092\textwidth]{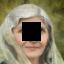}
    \includegraphics[width=0.092\textwidth]{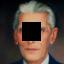}
    \includegraphics[width=0.092\textwidth]{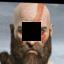}\\	
    \raisebox{0.25in}{\rotatebox[origin=t]{90}{\scriptsize DCGAN}}
    \includegraphics[width=0.092\textwidth]{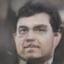}
    \includegraphics[width=0.092\textwidth]{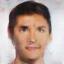}
    \includegraphics[width=0.092\textwidth]{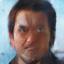}
    \includegraphics[width=0.092\textwidth]{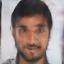}
    \includegraphics[width=0.092\textwidth]{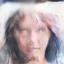}
    \includegraphics[width=0.092\textwidth]{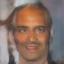}
    \includegraphics[width=0.092\textwidth]{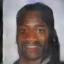}
    \includegraphics[width=0.092\textwidth]{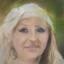}
    \includegraphics[width=0.092\textwidth]{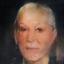}
    \includegraphics[width=0.092\textwidth]{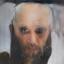}\\	
    \raisebox{0.25in}{\rotatebox[origin=t]{90}{\scriptsize Glow}}
    \includegraphics[width=0.092\textwidth]{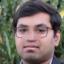}
    \includegraphics[width=0.092\textwidth]{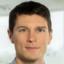}
    \includegraphics[width=0.092\textwidth]{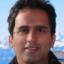}
    \includegraphics[width=0.092\textwidth]{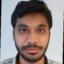}
    \includegraphics[width=0.092\textwidth]{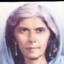}
    \includegraphics[width=0.092\textwidth]{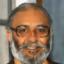}
    \includegraphics[width=0.092\textwidth]{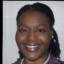}
    \includegraphics[width=0.092\textwidth]{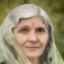}
    \includegraphics[width=0.092\textwidth]{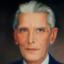}
    \includegraphics[width=0.092\textwidth]{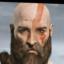}\\
    \raisebox{0.25in}{\rotatebox[origin=t]{90}{\scriptsize Truth}}
    \includegraphics[width=0.092\textwidth]{./images/celeba/transfer_faces/original/010.jpg}
    \includegraphics[width=0.092\textwidth]{./images/celeba/transfer_faces/original/012.jpg}
    \includegraphics[width=0.092\textwidth]{./images/celeba/transfer_faces/original/013.jpg}
    \includegraphics[width=0.092\textwidth]{./images/celeba/transfer_faces/original/015.jpg}
    \includegraphics[width=0.092\textwidth]{./images/celeba/transfer_faces/original/016.jpg}
    \includegraphics[width=0.092\textwidth]{./images/celeba/transfer_faces/original/017.jpg}
    \includegraphics[width=0.092\textwidth]{./images/celeba/transfer_faces/original/018.jpg}
    \includegraphics[width=0.092\textwidth]{./images/celeba/transfer_faces/original/023.jpg}
    \includegraphics[width=0.092\textwidth]{./images/celeba/transfer_faces/original/020.jpg}
    \includegraphics[width=0.092\textwidth]{./images/celeba/transfer_faces/original/021.jpg}\\
    \raisebox{0.25in}{\rotatebox[origin=t]{90}{\scriptsize Masked}}
    \includegraphics[width=0.092\textwidth]{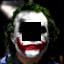}
    \includegraphics[width=0.092\textwidth]{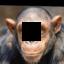}
    \includegraphics[width=0.092\textwidth]{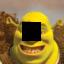}
    \includegraphics[width=0.092\textwidth]{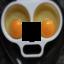}
    \includegraphics[width=0.092\textwidth]{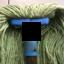}
    \includegraphics[width=0.092\textwidth]{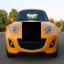}
    \includegraphics[width=0.092\textwidth]{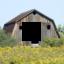}
    \includegraphics[width=0.092\textwidth]{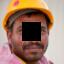}
    \includegraphics[width=0.092\textwidth]{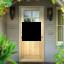}
    \includegraphics[width=0.092\textwidth]{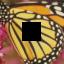}\\
    \raisebox{0.25in}{\rotatebox[origin=t]{90}{\scriptsize DCGAN}}
    \includegraphics[width=0.092\textwidth]{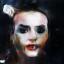}
    \includegraphics[width=0.092\textwidth]{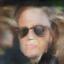}
    \includegraphics[width=0.092\textwidth]{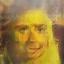}
    \includegraphics[width=0.092\textwidth]{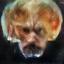}
    \includegraphics[width=0.092\textwidth]{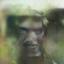}
    \includegraphics[width=0.092\textwidth]{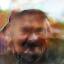}
    \includegraphics[width=0.092\textwidth]{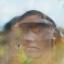}
    \includegraphics[width=0.092\textwidth]{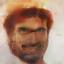}
    \includegraphics[width=0.092\textwidth]{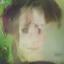}
    \includegraphics[width=0.092\textwidth]{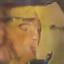}\\
    \raisebox{0.25in}{\rotatebox[origin=t]{90}{\scriptsize Glow}}
    \includegraphics[width=0.092\textwidth]{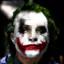}
    \includegraphics[width=0.092\textwidth]{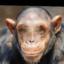}
    \includegraphics[width=0.092\textwidth]{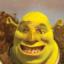}
    \includegraphics[width=0.092\textwidth]{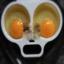}
    \includegraphics[width=0.092\textwidth]{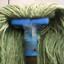}
    \includegraphics[width=0.092\textwidth]{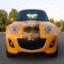}
    \includegraphics[width=0.092\textwidth]{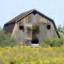}
    \includegraphics[width=0.092\textwidth]{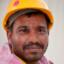}
    \includegraphics[width=0.092\textwidth]{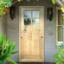}
    \includegraphics[width=0.092\textwidth]{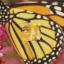}
    
        \caption{Image inpainiting results on out-of-distribution images. Masked images are recovered under the DCGAN prior and the Glow prior. Recoveries under the DCGAN prior are skewed and blurred whereas Glow prior leads to sharper and coherent inpainted images. For both Glow and DCGAN, we set $\gamma = 0$.}
    \label{fig:inpainting-OOD-supp}
\end{figure}

\newpage

\subsection{Discussion}
We provide evidence that random perturbations in image space induce larger changes in $z$ than comparable natural perturbations in image space.  Figure \ref{fig:perturbation-size-vs-direction} shows a plot of the norm of the change in image space, averaged over 100 test images, as a function of the size of a perturbation in latent space.  Natural directions are given by the interpolation  between the latent representation of two test images.  For the denoising problem, this difference in sensitivity indicates that the optimization algorithm might obtain a larger decrease in $\|z\|$ by an image modification that reduces unnatural image components than by a correspondingly large modification in a natural direction.

\begin{figure}[h!]
    \centering
    \includegraphics[width=0.4\textwidth]{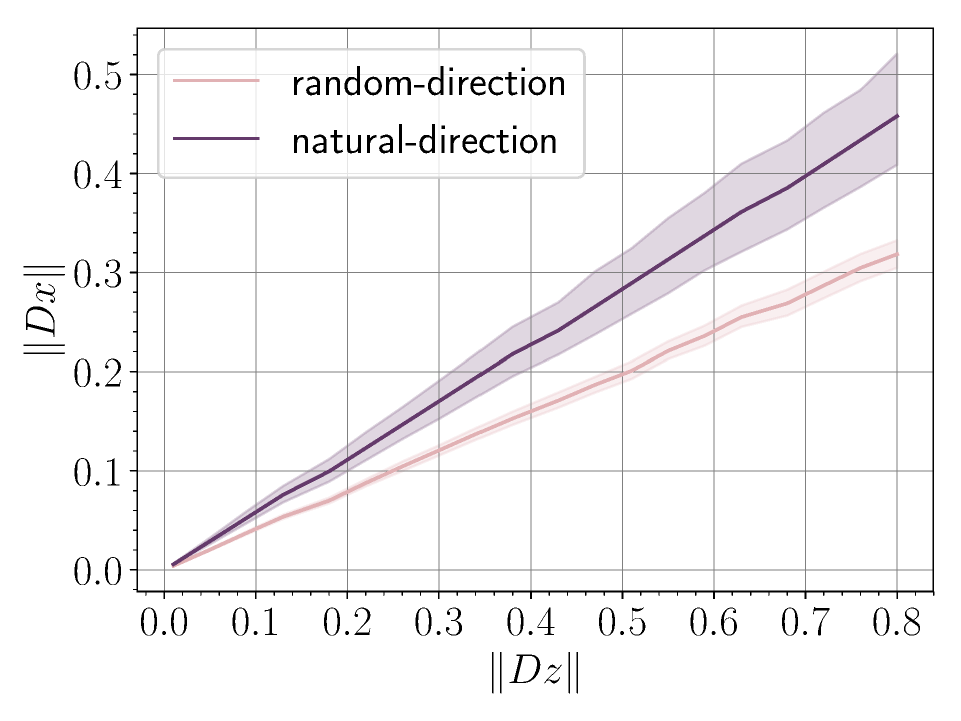}
    \caption{The magnitude of the change in image space  as a function of the size of a perturbation in latent space.  Solid lines are the mean behavior and shaded region depicts $95\%$ confidence interval.} 
    \label{fig:perturbation-size-vs-direction}
\end{figure} 

To further illustrate this point, we investigate in Figure \ref{fig:jacobian} the decay of the singular values of the Glow model's Jacobian for random points.
\begin{figure}
    \centering
    \includegraphics[width=0.6\textwidth]{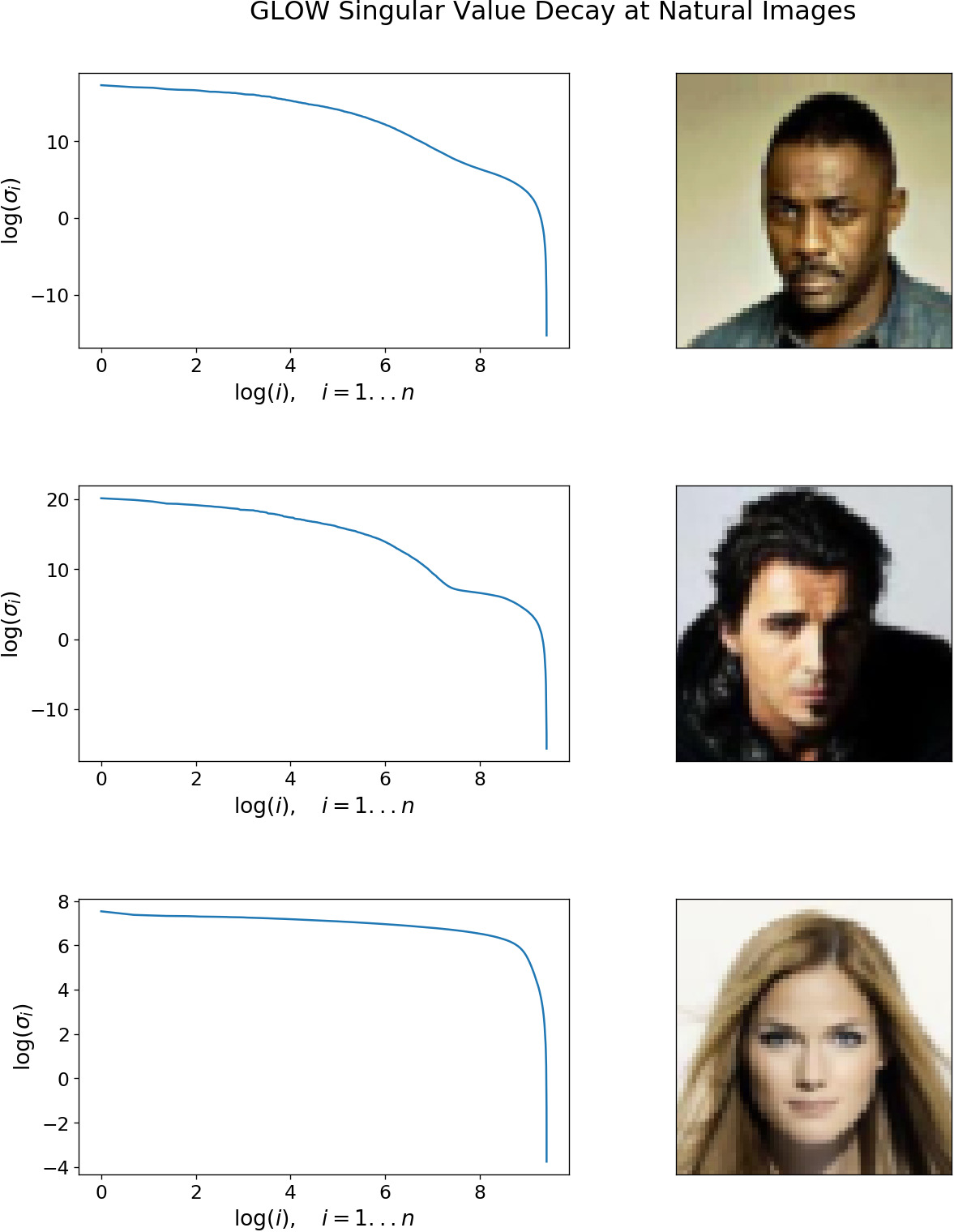}
    \caption{Log-log plots of the decay of the singular values of the trained Glow model's Jacobian for three random CelebA images.}
    \label{fig:jacobian}
\end{figure}

\clearpage


\end{document}